\documentclass{article}
\pdfoutput=1

\usepackage{microtype}
\usepackage{graphicx}
\usepackage{subfigure}
\usepackage{booktabs} 
\usepackage{multirow}
\usepackage{multicol}
\usepackage{hhline}
\usepackage{makecell}
\usepackage{graphicx}
\usepackage{enumitem}
\usepackage{algorithmic}
\usepackage{algorithm}
\usepackage{amssymb} 
\usepackage{amsthm}
\newtheorem{assumption}{Assumption}
\usepackage{booktabs}
\usepackage{varwidth}
\usepackage{pifont}

\usepackage{fontawesome5}
\usepackage{xcolor}
\usepackage{subcaption}
\usepackage{svg}
\usepackage{xcolor}
\usepackage{amsmath,amssymb,amsfonts}
\definecolor{lightgray}{RGB}{220, 220, 220} 
\usepackage{colortbl}

\newcommand{\ours}{DISCRET}

\newcommand{\otherx}[1]{x^{*}_{#1}}
\newcommand{\othert}[1]{t^{*}_{#1}}
\newcommand{\others}[1]{s^{*}_{#1}}
\newcommand{\othery}[1]{y^{*}_{#1}}

\newcommand{\ihdp}{IHDP}
\newcommand{\ihdpc}{IHDP-C}
\newcommand{\tcga}{TCGA}
\newcommand{\news}{News}

\newcommand{\eeec}{EEEC}
\newcommand{\uganda}{Uganda}

\let\oldding\ding
\renewcommand{\ding}[2][1]{\scalebox{#1}{\oldding{#2}}}
\newcommand{\cmark}{\ding{51}} %
\newcommand{\xmark}{\ding{55}} %
\usepackage{amsmath,amsfonts,bm}

\def\eqref#1{equation~\ref{#1}}

\def\1{\bm{1}}

\def\vw{{\bm{w}}}

\DeclareMathAlphabet{\mathsfit}{\encodingdefault}{\sfdefault}{m}{sl}
\SetMathAlphabet{\mathsfit}{bold}{\encodingdefault}{\sfdefault}{bx}{n}

\usepackage{hyperref}

\usepackage[accepted]{icml2024}

\usepackage{amsmath}
\usepackage{amssymb}
\usepackage{mathtools}
\usepackage{amsthm}

\usepackage[capitalize,noabbrev]{cleveref}

\theoremstyle{plain}
\newtheorem{theorem}{Theorem}[section]

\newtheorem{lemma}[theorem]{Lemma}

\theoremstyle{definition}

\theoremstyle{remark}

\usepackage[textsize=tiny]{todonotes}

\icmltitlerunning{\ours: Synthesizing Faithful Explanations For Treatment Effect Estimation}

\begin{document}

\twocolumn[
\icmltitle{\ours: Synthesizing Faithful Explanations For Treatment Effect Estimation}



\icmlsetsymbol{equal}{*}

\begin{icmlauthorlist}
\icmlauthor{Yinjun Wu}{equal,yyy}
\icmlauthor{Mayank Keoliya}{equal,comp}
\icmlauthor{Kan Chen}{sch}
\icmlauthor{Neelay Velingker}{comp}
\icmlauthor{Ziyang Li}{comp}
\icmlauthor{Emily J Getzen}{med}
\icmlauthor{Qi Long}{comp,med}
\icmlauthor{Mayur Naik}{comp}
\icmlauthor{Ravi B Parikh}{med}
\icmlauthor{Eric Wong}{comp}
\end{icmlauthorlist}

\icmlaffiliation{yyy}{School of Computer Science, Peking University, Beijing, China}
\icmlaffiliation{comp}{Department of Computer and Information Science, University of Pennsylvania, Philadelphia, PA, United States}
\icmlaffiliation{sch}{School of Public Health, Harvard University, Boston, MA, United States}
\icmlaffiliation{med}{Perelman School of Medicine, University of Pennsylvania, Philadelphia, PA, United States}

\icmlcorrespondingauthor{Yinjun Wu}{wuyinjun@pku.edu.cn}
\icmlcorrespondingauthor{Mayank Keoliya}{mkeoliya@seas.upenn.edu}

\icmlkeywords{Machine Learning, ICML}

\vskip 0.3in
]




\printAffiliationsAndNotice{\icmlEqualContribution}

\begin{abstract}
Designing faithful yet accurate AI models is challenging, particularly in the field of individual treatment effect estimation (ITE).
ITE prediction models deployed in critical settings such as healthcare should ideally be (i) accurate, and (ii)  provide faithful explanations.  However, current solutions are inadequate: state-of-the-art black-box models do not supply explanations, post-hoc explainers for black-box models lack faithfulness guarantees, and self-interpretable models greatly compromise accuracy. To address these issues, we propose \ours, a self-interpretable ITE framework that synthesizes faithful, rule-based explanations for each sample. A key insight behind \ours\ is that explanations can serve dually as \textit{database queries} to identify similar subgroups of samples. We provide a novel RL algorithm to efficiently synthesize these explanations from a large search space.
We evaluate \ours\ on diverse tasks involving tabular, image, and text data. \ours\ outperforms the best self-interpretable models and has accuracy comparable to the best black-box models while providing faithful explanations. \ours\ is available at \url{https://github.com/wuyinjun-1993/DISCRET-ICML2024}.

\end{abstract}

\vspace{-0.2in}
\section{Introduction}\label{sec: intro}

Designing accurate and explainable AI models is a key challenge in solving a wide range of problems that require individualized explanations. In this paper, we tackle this challenge in the context of individual treatment effect (ITE) estimation. ITE quantifies the difference between one individual’s outcomes with and without receiving treatment. 
Estimating ITE is a significant problem not only in healthcare \citep{basu2011estimating} but also in other domains such as linguistics \citep{pryzant2021causal, feder2021causalm} and poverty alleviation \citep{JJD-Confounding, JJD-Heterogeneity}. 
A large body of literature has investigated accurately estimating ITE using various machine learning architectures, including GANs \citep{yoon2018ganite} and transformers \citep{zhang2022exploring}, among others \citep{shalit2017estimating, liu2022causalegm}. 



ITE prediction models deployed in critical settings
should ideally be (i) \textbf{accurate}, and (ii) provide \textbf{faithful explanations} in order to be trustable and usable. In this paper, we follow prior work on evaluating the faithfulness of explanations in terms of \textit{consistency}, which measures the degree to which samples with similar explanations have similar model predictions \citep{dasgupta2022framework, nauta2023explainableai}. 

Current solutions for predicting ITE are either accurate or faithful, but not both, as illustrated in the first two rows of Figure \ref{fig: motivating_example}. While self-interpretable models such as Causal Forest and others ~\citep{athey2019estimating,chen2023covariate} produce consistent explanations, they struggle to provide sufficiently accurate ITE estimations. On the other hand, while black-box models like transformers are typically the most accurate, explanations generated by post-hoc explainers, such as Anchor \citep{ribeiro2018anchors}, are not provably consistent.

\begin{figure*}
    \centering
    \includegraphics[width=\textwidth]{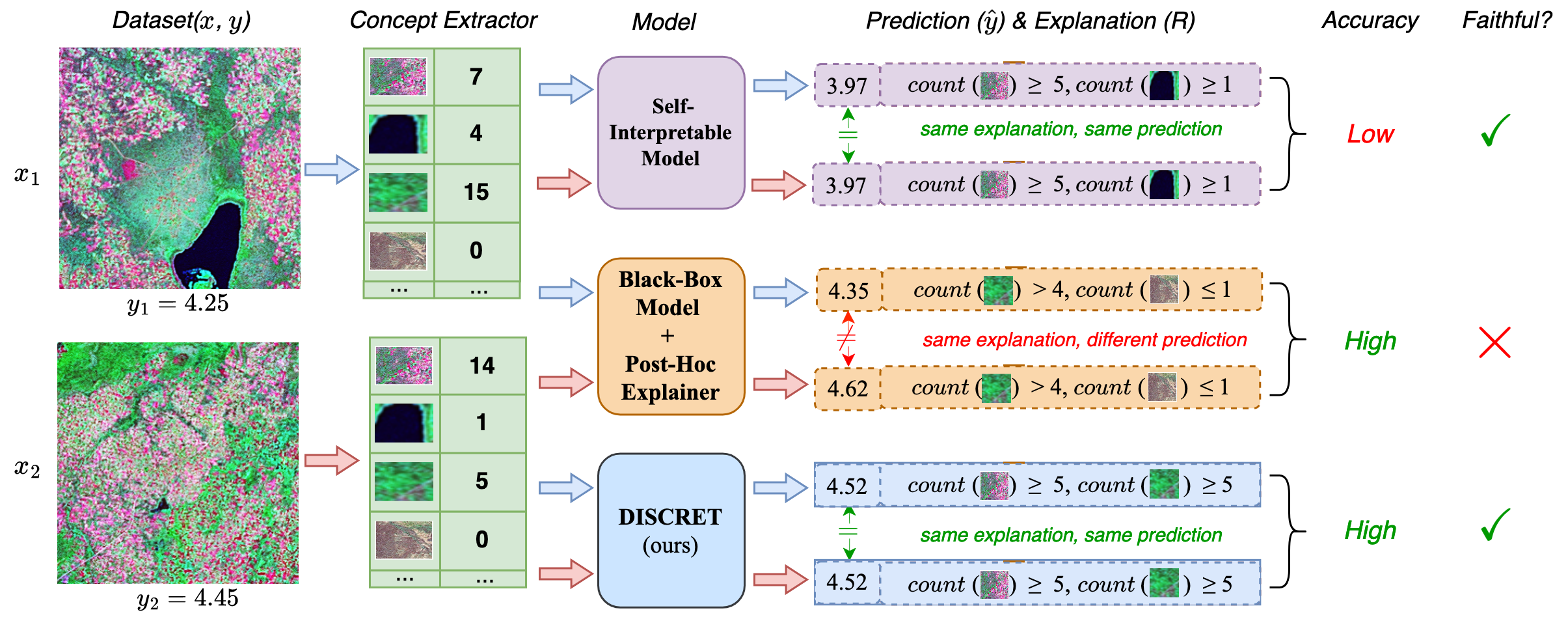}
    \vspace{-0.2in}
    \caption{Motivating examples from the Uganda dataset. We predict how providing economic aid (the treatment) helps to develop remote regions of the country (the outcome) via satellite images. The task is to estimate the ITE for each sample $x_1$ and $x_2$. \ours\ predicts that, because both images have several indicators of rich soil and urbanization, they will have similar ITE if given aid. Self-interpretable models such as Causal Forest \cite{athey2019estimating} produce \textit{consistent} ITE estimates (i.e., samples with same explanations have same model predictions, \textit{viz.} 3.97 and 3.97), but have poor accuracy ($\hat{ITE_{x_1}} \ll ITE_{x_1} = 4.25$). Black-box models such as TransTEE \cite{zhang2022exploring}, are accurate but do not produce similar predictions for samples $x_1$ and $x_2$ with similar explanations, when the explanations are sourced from post-hoc explainers such as Anchor \cite{ribeiro2018anchors}. \ours\ produces both consistent and accurate predictions.}
    

    \label{fig: motivating_example}
\end{figure*}

We therefore seek to answer the following central question:
    \emph{Is it possible to design a faithfully explainable yet accurate learning algorithm for treatment effect estimation?}
To this end, we propose \ours\footnote{\underline{DIS}covering \underline{C}omparable items with \underline{R}ules to \underline{E}xplain \underline{T}reatment Effect}, the first provably-faithful, deep learning based ITE prediction framework.
Given a sample $x$, \ours\ follows prior work and estimates ITE by computing the average treatment effect (ATE) of samples that are similar to $x$. However, in contrast to prior methods that discover similar samples through statistical matching \citep{anderson1980stratification, chen2023testing} or clustering \citep{xue2023assisting}, \ours\ finds similar samples by (i) synthesizing a logical rule that describes the key features of sample $x$ (and hence \textit{explains} the subgroup the sample belongs to) and then (ii)  
evaluating this rule-based explanation on a database of training samples (see Figure \ref{fig:alg_pipeline} for our pipeline).  
As shown in Figure \ref{fig: motivating_example}, \ours\ produces consistent explanations for samples with similar predictions; in fact, it is guaranteed to be consistent by construction, as we show later.


How does \ours\ synthesize rules which correctly group similar samples, and thus lead to accurate predictions? Learning to synthesize rules is challenging since the execution of database queries is non-differentiable and thus we cannot compute an end-to-end loss easily. 
To address this issue, we design a deep reinforcement learning algorithm with a novel and tailored reward function for dynamic rule learning. We also state the theoretical results of the convergence of \ours\ under some mild conditions suggesting if the ground-truth explanations are consistent, then our training algorithm can always discover them.


Due to the widely recognized trade-offs between interpretability and prediction performance~\citep{dziugaite2020enforcing}, \ours\ slightly underperforms the state-of-the-art black-box models \citep{zhang2022exploring}.  In addressing this, we found that regularizing the training loss of black-box models such as TransTEE to penalize discrepancy with \ours\ predictions yields new state-of-the-art models.

We evaluate the capabilities of \ours\ through comprehensive experiments spanning four tabular, one image, and one text dataset, covering three different types of treatment variables. For tabular data, among others, we use the \ihdp\ dataset \citep{hill2011bayesian} which tracks cognitive outcomes of premature infants. Other datasets used are \tcga\ (tabular) \citep{weinstein2013cancer}, \ihdpc\ (tabular), Uganda satellite images for estimating poverty intervention (image), and the Enriched Equity Evaluation Corpus (text). Notably, our approach outperforms all self-interpretable methods, including by 34\% on IHDP, is comparable to the accuracy of black-box models, and produces more faithful explanations than post-hoc explainers. In addition, regularizing the state-of-the-art black-box models with \ours\ reduces their ITE prediction error across tasks, including by 18\% on TCGA.


Our contributions can be summarized as follows:
\begin{enumerate}
    \item 
    We introduce \ours, a self-interpretable framework that synthesizes faithful rule-based explanations, and apply it to the treatment effect estimation problem.

    \item 
    We present a novel Deep Q-learning algorithm to 
    automatically learn these rule-based explanations, and supplement it with theoretical results.
    \item 
    We conduct an extensive empirical evaluation that demonstrates
    that \ours\ outperforms existing self-interpretable models and is comparable to black-box models across tabular, image, and text datasets spanning a diverse range of treatment variable types. Moreover, regularizing the state-of-the-art black-box models with \ours\ further reduces their prediction error.

\end{enumerate}
\vspace{-0.05in}
\section{Preliminaries}\label{sec: prelim}

\begin{figure*}[t]
  \centering
  \includegraphics[width=\textwidth]{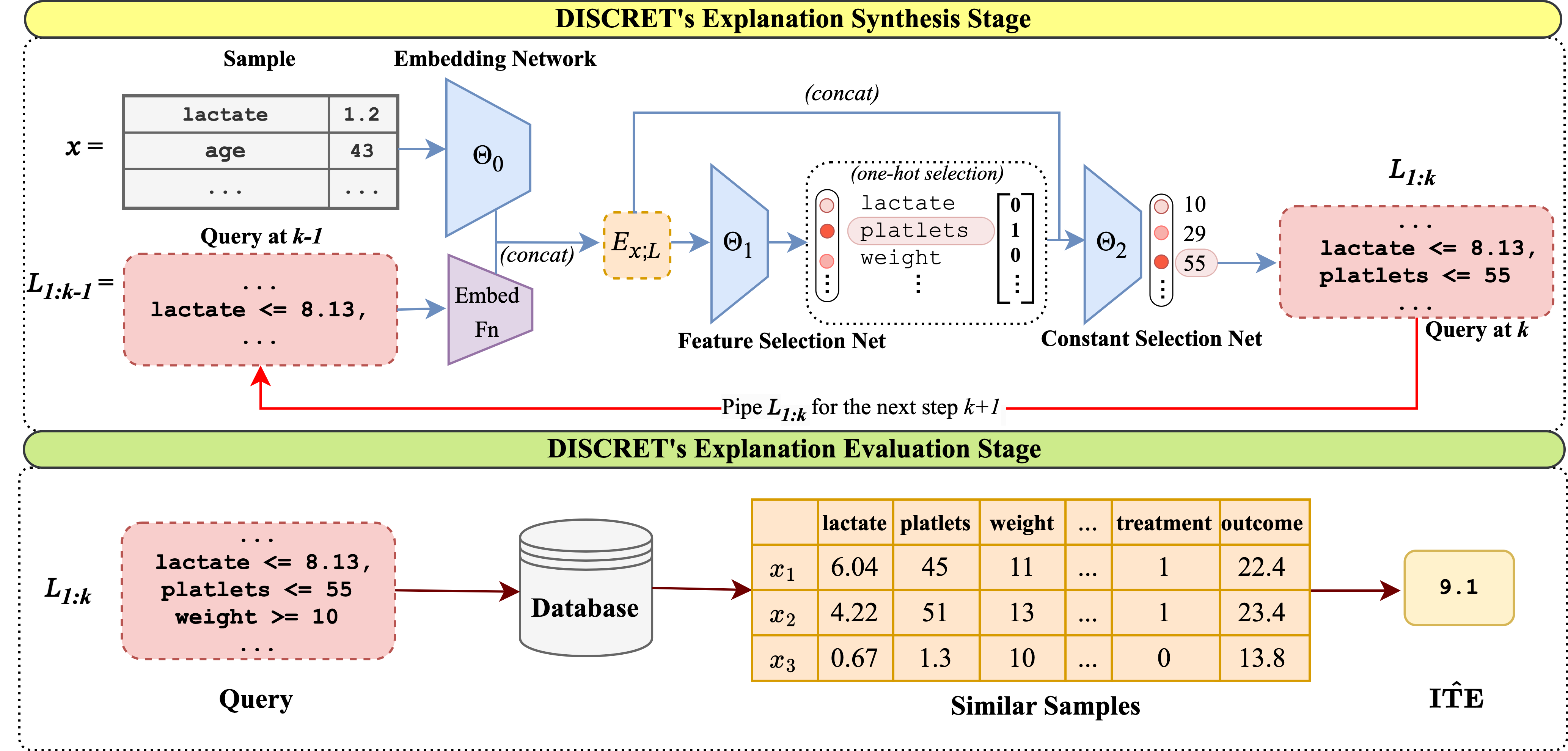} 

  \caption{Illustration of \ours\ on the IHDP dataset, which tracks premature infants. Given a sample $x$, \ours\  synthesizing an explanation $L_{1:k}$ where it iteratively constructs each literal in the explanation. In particular, \ours\ (i) embeds the given sample and any previously generated literals ($\Theta_0$),  (ii) passes the embedding to the feature selection network ($\Theta_1$) to pick a feature, and then (iii) passes the embedding and selected feature to the constant selection network ($\Theta_2$) to get a thresholding constant. The operator is auto-assigned based on the feature and sample. \ours\ executes this explanation on the database to find relevant samples, which are used (i) during training to compute a reward function for $\Theta_0, \Theta_1$ and $\Theta_2$, and (ii) during testing to calculate the ITE.}
  
  \label{fig:alg_pipeline}
  \vspace{-0.05in}
\end{figure*}


\subsection{Individual Treatment Effect (ITE) Estimation}
Suppose each sample consists of (i) the pre-treatment covariate variable $X$, (ii) the treatment variable $T$, (iii) a dose variable $S$ associated with $T$, and (iv) observed outcome $Y$ under treatment $T$ and dose $S$. We embrace a versatile framework throughout this study, where $T$ can take on either discrete or continuous values, $S$ is inherently continuous but can be either present or absent, $Y$ can be discrete or continuous, and $X$ may incorporate structured features as well as unstructured features, such as text or image data. In the rest of the paper, we primarily explore a broadly studied setting where $Y$ is a continuous variable, $T$ is a binary variable ($T=1$ and $T=0$ represent treated and untreated respectively) and there is no dose variable. The goal is to estimate individual treatment effects (ITE), i.e., the difference of outcomes with $T=1$ and $T=0$. Typically, the average treatment effect ($\text{ATE}$), the average of $\text{ITE}$ across all samples (i.e., $\text{ATE} = \mathbb{E}[\text{ITE}]$) is reported. Generalizations to other settings are provided in Appendix \ref{sec: general_treatment_var}. 

Beyond the treatment effect definitions, the propensity score, represented as the probability of treatment assignment $T$ conditioned on the observed covariates $X$, often plays a pivotal role in regularizing the treatment effect estimation. This propensity score is denoted as $\pi(T| X)$.

Unlike conventional prediction tasks, 
we are unable to directly observe the counterfactual outcomes during training, rendering the ground-truth treatment effect typically unavailable. To address this challenge and ensure the causal interpretability of our estimated treatment effect, we adhere to the standard assumptions proposed by \citet{rubin1974estimating}, which are formulated in Appendix \ref{sec:te_assumption}.



\vspace{-0.05in}
\subsection{Syntax of Logic Rules}
We assume that the covariate variable $X$ is composed of $m$ features, $X_1,X_2,\dots,X_m$, which can be categorical or numeric attributes from tabular data or pre-processed features extracted from text data or image data. We then build logic rule-based explanations upon those features to construct our treatment effect estimator. 
Those logic rules are assumed to be in the form of $K$ disjunctions of multiple conjunctions, i.e., $R_1 \lor R_2 \lor \cdots \lor R_H$ where each $R_i$ is a conjunction of $K$ literals: $l_{i1} \land l_{i2} \land l_{i3} \land \cdots \land l_{iK} $.
Each $l_{ij} (j=1,2,\dots)$ represents a literal of the form $l_{ij} = (A\ {op}\ c)$, where $A \in \{X_1,X_2,...,X_m\}$;
${op}$ is equality or inequality for categorical attributes, and ${op} \in \{<, >, =\} $ for numeric attributes; and $c$ is a constant.







\vspace{-0.05in}
\section{The \ours\ Framework}

Given a database $\mathcal{D}$ of individual samples with their covariate variables, and their ground-truth outcomes under treatment $T$ and dose $S$, 
we want to estimate the treatment effect on a new sample $x$. To do so, \ours\ consists of a two-step process: (i) \textit{explanation synthesis} where a rule-based explanation $R_x$ is synthesized for the given sample $x$, such that $R_x$ captures pertinent characteristics about the sample, and then (ii) \textit{explanation evaluation}, where a subgroup of similar samples $R_x(\mathcal{D}) \subseteq \mathcal{D}$ satisfying the explanation is selected from $\mathcal{D}$. Finally, the predicted ITE is computed over this subgroup $R_x(\mathcal{D})$. 

This section first outlines these two steps of \ours\ (\S \ref{sec: expl_synthesis} and \S \ref{sec: expl_evaluation}, Fig. \ref{fig:alg_pipeline}). We then explain the training algorithm (\S \ref{sec: training}). Additionally, we show how \ours\ can be employed to regularize state-of-the-art deep learning models for maximal performance (\S \ref{sec: regularization}). 

\vspace{-0.05in}
\subsection{Explanation Synthesis}
\label{sec: expl_synthesis}

\subsubsection{Overview}\label{sec: overview}
\ours 's explanation synthesizer consists of a set of three models, $\Theta=\{\Theta_0$, $\Theta_1$, $\Theta_2\}$. $\Theta_0$ is a backbone model for encoding features, $\Theta_1$ is a feature-selector, and $\Theta_2$ a thresholding constant selector for features. 
Note that $\Theta_0$ can be any encoding model, such as the encoder of the TransTEE model \cite{zhang2022exploring}. $\Theta_0$ can be optionally initialized with a pre-trained phase (see Appendix \ref{sec: pre-training phase}) and can be frozen or fine-tuned during the training phase.

Gven a sample $x$, and models $\Theta_0$, $\Theta_1$ and $\Theta_2$, we want to synthesize a conjunctive rule $R_x$ which takes the form of $R_x\ \text{:-}\ l_{1} \land l_{2} \land l_{3} \land \cdots \land l_{K}$. We synthesize $R_x$ by generating 
$l_{k}, k=1,2,\cdots,K$ recursively, where each $l_k$ takes the form $(A\ {op}\ c)$. Specifically, for each $l_k=A\ {op}\ c$, we select a feature $A$ using $\Theta_1$, a thresholding constant $c$ using $\Theta_2$, and an operator ${op}$ based on $x$, $A$ and $c$.
Before illustrating how to synthesize these rules during the inference phase in \S \ref{sec: rule_generation}, we take a light detour to describe some desired properties for them in \S \ref{sec: desired}.

\subsubsection{Desired Properties of Explanations}\label{sec: desired}

We state four desired properties of a rule-based explanation, which guide the design of \ours. We will refer to these properties in \S \ref{sec: rule_generation} and \S \ref{sec: training}.
\begin{enumerate}[leftmargin=*,nolistsep, nosep]
    \item \textbf{Local interpretability}: We aim to synthesize a rule-based explanation $R_x$ for \textit{each} individual sample $x$ rather than for a population of samples. Thus, explanations may differ for different samples. 
    \item \textbf{Satisfiability}: For any rule $R_x$ generated for a given sample $x$, $x$'s features must satisfy $R_x$. This guarantees that the sample $x$ and any samples retrieved by $R_x$ share the same characteristics.
    \item \textbf{Low-bias}: We expect that $R_x$ can retrieve a set of similar samples so that the bias between the estimated ATE over them and the ground-truth ITE is as small as possible.
    \item {\bf Non-emptiness}: There should be at least one sample from the database whose covariates satisfy $R_x$. In addition, for those samples satisfying $R_x$, their treatment variables should cover all essential treatment values for treatment effect estimations, e.g., containing both treated and untreated units in binary treatment settings. 
\end{enumerate}

\subsubsection{Rule Generation}\label{sec: rule_generation}
The generation of the rule $R_x$ during inference is straightforward. At each round $k$, we encode the features $E_x$ and the so-far generated rule $L_{1:k-1}(=l_{1} \land l_{2} \land l_{3} \land \cdots \land l_{k-1})$ and select a feature $A_k$ from $\Theta_1$ by (see Appendix \ref{sec: encoding_rules} for details). For each feature $A_k$, we select a thresholding constant $c$ and operator ${op}$ to form literal $l_k$. Selection of  $c$ and ${op}$ depends on the type of $A_k$.

\textbf{Categorical Features.} If $A$ is a categorical attribute, then we assign $c = x[A]$, where $x[A]$ is the value of attribute A in sample $x$; and we assign ${op}$ as $=$, which guarantees the \textbf{satisfiability} of $R_x$ on $x$. 

\textbf{Numeric Features.} If $A$ is a numeric attribute, we first discretize the range of $A$ into bins, and query $\Theta_2$ to choose a bin $C_j$. As suggested in Figure \ref{fig:alg_pipeline},  $\Theta_2$ takes the encoding of the covariates and $L_{1:k-1}$, and the one-hot encoding of feature $A$ as the model input. After the feature $A$ and the constant $c$ are identified, the operator ${op}$ is then deterministically chosen by comparing the value $x[A]$ and $c$. If $x[A]$ is greater than $c$, then ${op}$ is assigned as $\geq$, and as $\leq$ otherwise, thus again guaranteeing the {\bf satisfiability} of the rule $R_x$. 

In addition, we observe that the samples retrieved by the rule $R_x$ may not contain all essential treatment values for treatment effect estimations, thus violating the {\bf Non-emptiness}. To address this issue, we keep track of the retrieved samples for each $L_{1:k}(k=1,2,\dots,K)$ and whenever the addition of one literal $l_{k+1}$ leads to the violation of the {\bf Non-emptiness} property, we stop the rule generation process early and return $L_{1:k}$ as $R_x$. 

To produce multiple disjunctions with \ours, multiple literals are generated simultaneously at each round, each of which is assigned to one disjunction respectively (see Appendix \ref{sec: gen_disjunctive}).

\subsection{Explanation Evaluation}
\label{sec: expl_evaluation}

As Figure \ref{fig:alg_pipeline} shows, given a sample $x$ (e.g., a patient) with $(X,T,S,Y)$, and a rule $R_x$ (i.e., $L_{1:k}$ in Figure \ref{fig:alg_pipeline}), we evaluate the rule $R_x$ on a database $\mathcal{D}$ to retrieve a subgroup of similar samples, which is denoted by $R_x(\mathcal{D}) = \{(\otherx{i}, \othert{i}, \others{i}, \othery{i})\}_{i=1}^n$. 

\textbf{ITE Estimation.} The ITE of the sample $x$ is then estimated by computing the average treatment effect (ATE) estimated within this subgroup. In this paper, we take the empirical mean by default for estimating ATE of $R_x(\mathcal{D})$, i.e., $\hat{y}(1) - \hat{y}(0)$, in which $\widehat{y}(t), (t=0,1)$ denotes the estimated outcome calculated with the following formula:

\vspace{-2em}
\begin{small}
\begin{align}
    \begin{split}
        \widehat{y}(t) = \frac{1}{\sum \mathbb{I}({\othert{i}=t})}\sum\mathbb{I}({\othert{i}=t})\cdot \othery{i}
    \end{split}
\end{align} 
\end{small}
\vspace{-1.3em}

We also estimate the propensity score for discrete treatment variables by simply calculating the frequency of every treatment within $R_x(\mathcal{D})$: $\widehat{\pi}(T=t| X=x) = \sum\mathbb{I}(\othert{i}=t)/{|R_x(\mathcal{D})|}$.

\subsection{RL-based Training}\label{sec: training}
We train $\Theta$ to satisfy the desired properties mentioned in \S \ref{sec: desired}. In particular, to preserve the {\bf low-bias} property, we need to guide the generation of rules such that the estimated ITE is as accurate as possible. However, a key difficulty in training $\Theta$ is the \textbf{non-differentiability} arising from the explanation evaluation step (\S \ref{sec: expl_evaluation}), i.e. evaluating $R_x$ on our database. We overcome this issue by formulating the model training as a deep reinforcement learning (RL) problem and propose to adapt the Deep Q-learning (DQL) algorithm to solve this problem. Briefly, we define a reward function over the selected subgroup of samples $R_x(\mathcal{D})$, and use it to learn the RL-policy.

We first map the notations from \S \ref{sec: overview} to classical RL terminology. 
An RL agent takes one {\it action} at one {\it state}, and collects a {\it reward} from the environment, which is then transitioned to a new state. 
In our rule learning setting, a {\it state} is composed of the covariates $x$ and the generated literals in the first $k-1$ rounds, $L_{1:k-1}$. 
With $x$ and $L_{1:k-1}$, the model $\Theta_1$ and $\Theta_2$ collectively determine the $k_{th}$ literal, $l_k$, which is regarded as one {\it action}. 
Our goal is then to learn a policy parameterized by $\Theta$, which models the probability distribution of all possible $l_k$ conditioned on the state $(x, L_{1:k-1})$, such that the value function calculated over all $K$ rounds is maximized:

\vspace{-1.9em}
\begin{small}
\begin{align}\label{eq: value_function}
    \begin{split}
        V_{1:K} = \sum\nolimits_{k=1}^K r_k \gamma^{k-1},
    \end{split}
\end{align} \par
\end{small}
\vspace{-1.3em}

in which $\gamma$ is a discounting factor. 
Note that 
there are only $K$ horizons/rounds in our settings since the number of conjunctions in the generated rules is limited. To bias rule generation towards accurate estimation of ITE, we expect that the value function $V_{1:K}$ reflects how small the ITE estimation error is. However, since the counterfactual outcomes are not observed in the training phase, we therefore use the errors of the observed outcomes as a surrogate of the ITE estimation error. Also, we give a zero reward to the case where the retrieved subgroup, $L_{1:K}(\mathcal{D})$, violates the {\bf non-emptiness} property.
As a result, $V_{1:K}$ is formulated as

\vspace{-1.9em}
\begin{small}
    \begin{align}\label{eq: reward_function}
        \begin{split}
        V_{1:K} = e^{-\alpha(y - \widehat{y}_{1:K})^2}\cdot \mathbb{I}( L_{1:K}(\mathcal{D})\text{ is non-empty}),            
        \end{split}
    \end{align} \par 
\end{small}

in which $\widehat{y}_{1:K}$ represents the estimated outcome by using the generated rule composed of literals $L_{1:K}$ and $\alpha$ is a hyper-parameter. 
As a consequence, the reward collected at the $k_{th}$ round of generating $l_k$ becomes $r_k=(V_{1:k} - V_{1:k-1})/\gamma^{k-1}$. We further discuss how to automatically fine-tune the hyper-parameter $\alpha$ and incorporate the propensity score defined in \S \ref{sec: expl_evaluation} for regularization in Appendix \ref{sec: reward_func_additional}. 
\par

Next, to maximize the value function $V_{1:K}$, we employ Deep Q-learning (DQL) \citep{mnih2013playing} to learn the parameter $\Theta$. To facilitate Q learning, we estimate the Q value 
with the output logits of the models given a state $(x, L_{1:k-1})$ and an action $l_k$.  Recall that since \ours\ can generate consistent explanations by design, we can show that if $\Theta_0$ is an identity mapping and $\Theta_1$ is a one-layer neural network, the following theorem holds: 

\begin{theorem} \label{thm: main}
Suppose we have input data $\{(x_i, t_i, s_i, y_i)\}_{i=1}^N$ where $x_i \in \mathbb{R}^m$ and discrete, $t_i \in \mathbb{R}, s_i \in \mathbb{R}$, and $y_i \in \mathbb{R}$, then the $\hat{ITE}_x$ obtained from \ours\ converges to zero generalization error with probability 1 for ITE estimation (i.e. $(ITE_x - \hat{ITE}_x)^2 \rightarrow 0$ w.p. 1) for any fixed $K \leq m$ over the dataset with all discrete features under the data generating process $y = f (\mathcal{X}_K ) + c \cdot t + \epsilon$, where $\mathcal{X}_K \subseteq \{X_1, X_2, \cdots, X_m \}, c \in \mathbb{R}, t$ is the treatment assignment, and $\epsilon \sim \mathcal{N}(0, \sigma^2)$ for some $\sigma > 0$. 
\end{theorem}

Intuitively, Theorem \ref{thm: main} suggests if the ground-truth explanations are consistent, then our training algorithm can perfectly discover them. We prove the theorem and explain our algorithm in detail in Appendix \ref{sec: addition_tech}.

\subsection{Regularizing Black-box Models with \ours}\label{sec: regularization}
Due to the widely recognized trade-offs between model interpretability and model performance \citep{dziugaite2020enforcing}, self-interpretable models typically suffer from poorer performance than their neural network counterparts. To achieve a better balance between performance and interpretability, we further propose to regularize the prediction of black-box models with that of \ours. Since \ours\ also leverages part of the black-box model such as the encoder of TransTEE as the backbone $\Theta_0$, we thus obtain the predictions of black-box models by reusing $\Theta_0$. 
Specifically, starting from the encoded covariates $E_x$ generated by $\Theta_0$, 
we predict another outcome $\hat{y'}$ directly with $E_x$ adhering to the mechanism employed by state-of-the-art neural models. This prediction is then regularized by the predicted outcome $\widehat{y}_{1:K}$ by \ours\ as follows:

\vspace{-2em}
\begin{small}
\begin{align*}
    \widehat{y'}_{1:K} = (\hat{y'} + \lambda \widehat{y}_{1:K})/(1 + \lambda),
\end{align*} \par
\end{small}
\vspace{-1.2em}

in which $\lambda$ is a hyperparameter for controlling the impact of $\hat{y}'$. Afterward, $\widehat{y}_{1:K}$ is replaced with $\widehat{y'}_{1:K}$ in  Equation \ref{eq: reward_function} or Equation \ref{eq: reward_function_2} for model training. In addition, to facilitate accurate $\hat{y'}$, we further minimize the loss involving $\hat{y'}$ and $y$ along with the Deep Q-learning loss.

\section{Experiments}\label{sec: experiments}
In this section, we aim to answer the following research questions about \ours:

{\small \textbf{RQ1:}} Does \ours\ produce faithful explanations? \\
{\small \textbf{RQ2:}} How does the accuracy of \ours\ perform compared to existing self-interpretable models and black-box models?

\vspace{-0.05in}
\subsection{Setup}\label{sec: exp_settings}

\textbf{Datasets.}
We evaluate across tabular, text, and image datasets, covering diverse categories of treatment variables. Specifically, 
we select 
\ihdp\ \citep{hill2011bayesian},
\tcga\ \citep{weinstein2013cancer}
\ihdpc\ (a variant of \ihdp), and \news\ for tabular setting, the Enriched Equity Evaluation Corpus (\eeec) dataset \citep{kiritchenko2018examining} for text setting and \uganda\ \citep{JJD-Heterogeneity, JJD-Confounding} dataset for the image setting.
We summarize the modality, categories of treatment and dose variables, and number of features for each dataset in Table \ref{tab: datasets}, with more details in Appendix \ref{sec: dataset}.

\begin{table}[]
\small
\begin{tabular}{llllll}
\toprule
Dataset & Type & Treatment & Dose & \# Features  \\
\midrule
IHDP    & Tabular  &    2     &  \xmark    &  25  \\        
TCGA    & Tabular  &    3     &  \cmark   & 4000   \\      
IHDP-C  & Tabular  &     cont.  &  \xmark    & 25   \\     
News    & Tabular  &     cont.  &   \xmark   & 2000    \\      
EEEC    & Text     &       2    &  \xmark    &  500   \\      
Uganda  & Image    &      2     &   \xmark   & 20   \\      
\bottomrule
\end{tabular}
\caption{Datasets used for evaluation (cont. means continous)}
\label{tab: datasets}
\vspace{-1.2em}
\end{table}

\textbf{Baselines}. We use extensive baselines for neural network models, self-interpretable models, and post-hoc explainers.

\textit{Neural network models.} For neural networks, we select the state-of-the-art models: TransTEE \citep{zhang2022exploring}, TVAE \citep{xue2023assisting}, 
Dragonnet \citep{shi2019adapting}, TARNet \citep{shalit2017estimating}, Ganite 
\citep{yoon2018ganite}, DRNet \citep{schwab2020learning}, and VCNet \citep{nie2020vcnet}.
Not all of these models support all categories of treatment variables, as discussed in Appendix \ref{sec: add_baseline}.
Also, since our regularization strategy can be regarded as the integration of two models through weighted summation, we compare our regularized backbone (TransTEE) against the integration of TransTEE and another top-performing neural network model (Dragonnet for \ihdp, \eeec, and \uganda\ dataset, VCNet for \tcga, DRNet for \ihdpc) 
in the same manner. 

\textit{Self-interpretable models.} We compare against classical self-interpretable models, e.g., 
Causal Forest \citep{athey2019estimating}, Bayesian Additive Regression Trees (BART) \cite{bart2010, bart2020},   
decision tree (DT), and random forests (RF), in which the latter two are integrated into 
R-learner \citep{nie2021quasi} for treatment effect estimation. 
We also adapt three general-purpose self-interpretable models to treatment effect estimation---ENRL \citep{shi2022explainable}, ProtoVAE \citep{gautam2022protovae}\footnote{ProtoVAE is designed for image data. We therefore only compare \ours\ against this method on the \uganda\ dataset.}, and Neural Additive Model (NAM) \citep{agarwal2021neural}, which generate rules, prototypes, and feature attributes as explanations respectively. For tree-based models among these methods, 
we maintain the same explanation complexity as \ours. For the sake of completeness we also conduct additional experiments to vary the complexity (e.g., the number of trees and tree depth) of all self-interpretable models, provided in Table \ref{tab: self-interpret-high-depth} in Appendix \ref{sec: varied_complexity}; \ours\ outperforms self-interpretable models even when they are configured to high complexity. 

\textit{Post-hoc explainers.} We apply several post-hoc explainers to the TransTEE model to evaluate the consistency of explanations. Thy include Lore \citep{guidotti2018local}, Anchor \citep{ribeiro2018anchors}, Lime \citep{ribeiro2016should}, Shapley values \citep{shrikumar2017learning}, and decision tree-based model distillation methods \citep{frosst2017distilling} (hereinafter referred to as Model Distillation). We enforce the complexity of these explanations to be the same as \ours\ for fair comparison.

\textbf{Evaluation metrics.}
We primarily evaluate faithfulness by measuring {\it consistency}, proposed by \citep{dasgupta2022framework};  we also measure \textit{sufficiency}, which is a generalization of consistency. Briefly, 
{\it consistency} quantifies how similar the model predictions are between samples with the same explanations, while {\it sufficiency} generalizes this notion to arbitrary samples \textit{satisfying} the same explanations (but not necessarily \textit{producing} the same explanations). Appendix \ref{sec: metrics} provides formal definitions of these two metrics.

We evaluate ITE estimation accuracy using different metrics for datasets to account for different settings.
For the datasets with binary treatment variables, by following prior studies \citep{shi2019adapting, shalit2017estimating}, we employ the absolute error in average treatment effect, i.e., $\epsilon_{ATE} = |\frac{1}{n}\sum_{i=1}^n ITE(x_i) - \frac{1}{n}\sum_{i=1}^n \widehat{ITE}(x_i)|$. Both in-sample and out-of-sample $\epsilon_{ATE}$ are reported, i.e., $\epsilon_{ATE}$ evaluated on the training set and test set respectively.
For the datasets with either continuous dose variables or continuous treatment variables, we follow \citep{zhang2022exploring} to report the average mean square errors $AMSE$ between the ground-truth outcome and predicted outcome on the test set.
 For the image dataset, \uganda, since there is no ground-truth ITE, we therefore only report the average outcome errors between the ground-truth outcomes and the predicted outcomes conditioned on observed treatments, i.e., $\epsilon_{\text{outcome}}=\frac{1}{n}\sum_{i=1}^n|y_i-\hat{y}_i|$.

\textbf{Configurations for \ours.} We consider two variants of \ours: vanilla \ours\ and backbone models regularized with \ours\ (denoted as \ours\ + TransTEE). For both variants, we perform grid search on the number of conjunctions, $K$, and the number of disjunctions, $H$, and the regularization coefficient $\lambda$, in which $K \in \{2,4,6\}$, $H \in \{1,3\}$ and $\lambda \in \{0,2,5, 8, 10\}$. 

\textbf{Extracting features from text and image data.} 
For text data, we employ the word frequency features such as ``Term Frequency-Inverse Document Frequency'' \citep{baeza1999modern}. 
For image data, we follow \citep{fel2023holistic} to extract interpretable concepts as the features, which we further discuss in Appendix \ref{sec: feature_extraction}.
Note that we only extract features for \ours\ and self-interpretable baselines such as Causal Forest while all neural network model-based baselines still take raw images or text data as input.

\begin{table*}[t]
\scriptsize
\centering
\setlength{\tabcolsep}{4pt}
\begin{tabular}{c|c|cc|cc|c|c|c}\toprule
\textbf{Modality} $\rightarrow$ & & \multicolumn{5}{c|}{\textbf{Tabular}} & \textbf{Text} & \textbf{Image} \\ \midrule
\textbf{Dataset} $\rightarrow$  &  & \multicolumn{2}{c|}{\ihdp} & \multicolumn{2}{c|}{\tcga} & \multicolumn{1}{c|}{\ihdpc} & \multicolumn{1}{c|}{\eeec} & \multicolumn{1}{c}{\uganda} \\ \midrule
\makecell{\textbf{Method} $\downarrow$} & 
\makecell{Self- \\ interp.?} & 
\makecell{$\epsilon_{\text{ATE}}$\\(In-sample)} & \makecell{$\epsilon_{\text{ATE}}$ \\ (Out-of-sample)} &\makecell{$\epsilon_{\text{ATE}}$\\(In-sample)} &\makecell{$\epsilon_{\text{ATE}}$\\ (Out-of-sample)}&AMSE & $\epsilon_{\text{ATE}}$ & $\epsilon_{\text{outcome}}$ \\ \hline 

\rowcolor{lightgray}Decision Tree & \cmark &0.693$\pm$0.028&0.613$\pm$0.045& 0.200$\pm$0.012&0.202$\pm$0.012 &22.136$\pm$1.741&0.014$\pm$0.016& 1,796$\pm$0.021\\
\rowcolor{lightgray}Random Forest & \cmark &0.801$\pm$0.039&0.666$\pm$0.055
&19.214$\pm$0.163 &19.195$\pm$0.163
&21.348$\pm$1.222&0.525$\pm$0.573& 1.820$\pm$0.013\\
\rowcolor{lightgray}NAM & \cmark &0.260$\pm$0.031&0.250$\pm$0.032 &- & - &24.706$\pm$0.756&0.152$\pm$0.041 & 1.710$\pm$0.098\\
\rowcolor{lightgray}ENRL & \cmark &4.104$\pm$1.060& 3.759$\pm$0.087&10.938$\pm$2.019 &10.942$\pm$2.019&24.720$\pm$0.985&-& 1.800$\pm$0.143\\
\rowcolor{lightgray}Causal Forest & \cmark &0.177$\pm$0.027&0.240$\pm$0.024&-&-&-&0.011$\pm$0.001&-\\
\rowcolor{lightgray}BART & \cmark &1.335$\pm$0.159&1.132$\pm$0.125&230.74$\pm$0.312&236.81$\pm$0.531&12.063$\pm$0.410&0.014$\pm$0.016& \underline{1.676$\pm$0.042}\\
\rowcolor{lightgray} \ours\ (ours) & \cmark  & 0.089$\pm$0.040 & 0.150$\pm$0.034 & 0.076$\pm$0.019 & {0.098$\pm$0.007} & 0.801$\pm$0.165 & \underline{0.001$\pm$0.017} & 1.662$\pm$0.136 \\ 
\hline
\hline
Dragonnet & \xmark & 0.197$\pm$0.023 &0.229$\pm$0.025 &-&-&-& 0.011$\pm$0.018& 1.709$\pm$0.127\\ 
TVAE & \xmark &3.914$\pm$0.065&3.573$\pm$0.087 &-&-&-&0.521$\pm$0.080& 49.55$\pm$2.38\\ 
TARNet & \xmark &0.178$\pm$0.028 &0.441$\pm$0.088 &1.421$\pm$0.078 &1.421$\pm$0.078& 12.967$\pm$1.781&0.009$\pm$0.018&1.743$\pm$0.135 \\ 
Ganite & \xmark &0.430$\pm$0.043 & 0.508$\pm$0.068&- &-&-&1.998$\pm$0.016& 1.766$\pm$0.024 \\ 
DRNet & \xmark &0.193$\pm$0.034 &0.433$\pm$0.080 &1.374$\pm$0.086 &1.374$\pm$0.085&11.071$\pm$0.994&0.008$\pm$0.018& 1.748$\pm$0.127\\ 
VCNet & \xmark &3.996$\pm$0.106 & 3.695$\pm$0.077&0.292$\pm$0.074&0.292$\pm$0.074&-&0.011$\pm$0.017&1.890$\pm$0.110 \\ 
TransTEE & \xmark &\textbf{0.081$\pm$0.009} &\underline{0.138$\pm$0.014} & \underline{0.070$\pm$0.010} &\underline{0.067$\pm$0.008}&\underline{0.112$\pm$0.008}&\underline{0.003$\pm$0.017}&1.707$\pm$0.158 \\ 
TransTEE + NN & \xmark & 0.224$\pm$0.022 & 0.300$\pm$0.035 &0.093$\pm$0.013&0.094$\pm$0.013& 0.363$\pm$0.033 & 0.006$\pm$0.008 & 2.001$\pm$0.425\\
\makecell{TransTEE + \ours\\ (ours)} & \xmark & \underline{0.082$\pm$0.009} &\textbf{0.120$\pm$0.014}&\textbf{0.058$\pm$0.010}&\textbf{0.055$\pm$0.009}&\textbf{0.102$\pm$0.007}&\textbf{0.001$\pm$0.017}& \textbf{1.662$\pm$0.136}\\\bottomrule 
\hline
\end{tabular}
\caption{
ITE estimation errors (lower is better). 
We \textbf{bold} the smallest estimation error for each dataset,  and \underline{underline} the second smallest one. We show that \ours\ outperforms self-interpretable models across all datasets, particularly on text ($\epsilon_{ATE} = 0.001$ for \ours\ v/s $0.011$ for causal forest). \ours\ is comparable to the performance of black-box models, with the exception of the \ihdpc\ dataset.  Regularizing black-box models with \ours\ (shown here as TransTEE + \ours)  outperforms \textit{all} models.}
\label{tab: quantative_res}
\vspace{-5px}
\end{table*}

\begin{figure*}[t]
  \centering
  \includegraphics[width=0.75\textwidth]{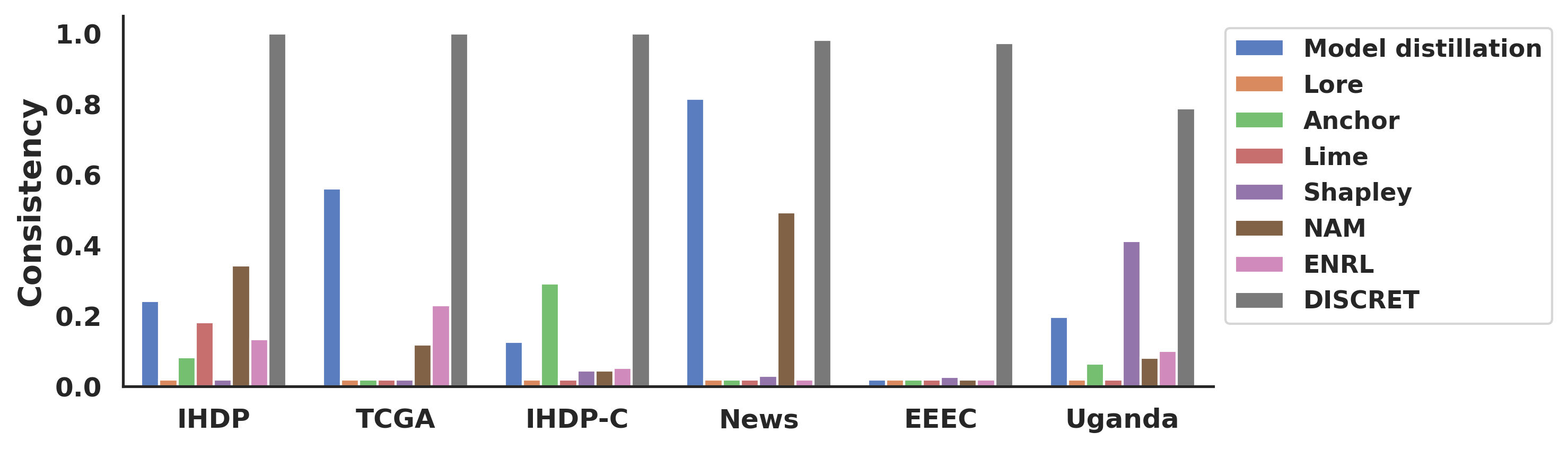}
  \vspace{-12px}
  \caption{Consistency scores (higher is better) for \ours\, and a black-box model (TransTEE) combined with a post-hoc explainer. Our results confirm that \ours\ produces faithful explanations, and importantly, show that post-hoc explanations are rarely faithful, as evidenced by low consistency scores across datasets. 
  }
  \vspace{-10px}
  \label{fig:consistency}
\end{figure*}

\subsection{RQ1: Faithfulness Evaluation on Explanations}
We evaluate the {\it consistency} and {\it sufficiency} of explanations produced by \ours, the state-of-the-art self-interpretable models, and the post-hoc explainers. 
For those explainers producing feature-based explanations, we also follow \citep{dasgupta2022framework} to discretize the feature importance scores, say, by selecting the Top-K most important features, for identifying samples with exactly the same explanations. 
For fair comparison, we evaluate the explanations generated w.r.t.~the same set of features extracted from NLP and image data. 

We graph the consistency scores in Figure \ref{fig:consistency}; full consistency scores are provided in Table \ref{tab: consistency} in Appendix \ref{sec:consistency_scores}. As Figure \ref{fig:consistency} indicates, \ours\ always achieves near 100\% consistency since the same explanations in \ours\ deterministically retrieve the same subgroup from the database, thus generating the same model predictions. In contrast, the baseline explanation methods generally have extremely low
consistency scores in most cases. We also include the sufficiency score results in Table \ref{tab: tabular_sufficiency}, which shows that \ours\ can still obtain higher sufficiency scores in most cases than other explanation methods. 

\subsection{RQ2: Accuracy Evaluation on ITE Predictions}
We include the ITE estimation results for tabular setting, NLP setting, and image setting in Table \ref{tab: quantative_res}. 
For brevity, the results on \news\ dataset 
are not reported in Table \ref{tab: quantative_res}, but are included in Table \ref{tab: news_dataset} 
in Appendix \ref{sec: news_res}.

As Table \ref{tab: quantative_res} shows, \ours\ outperforms all the self-interpretable methods, particularly on text ($\epsilon_{ATE} = 0.011$ for \ours\ v/s 0.0011 for causal forest). Compared to black-box models, \ours\ only performs slightly worse in most cases, and even outperforms them on the \uganda\ dataset. The outperformance is possibly caused by equivalent outcome values among most samples in this dataset as suggested by Figure \ref{fig:uganda_outcome} in Appendix \ref{sec: uganda_consistent}. Hence, consistent predictions (e.g., by \ours) between samples lead to a lower error rate. \ours\ underperforms TransTEE on \ihdpc, likely due to the complexity of the dataset; \ours\ still beats all other black-box models on this dataset.

Further, backbone models (TransTEE) regularized with \ours\ outperform the state-of-the-art neural network models, reducing their estimation errors by as much as 18\% (TCGA dataset.) Interestingly, for the \ihdp\ dataset, TransTEE outperforms its regularized version only on in-sample (i.e. training) error, but underperforms the regularized version when we consider out-of-sample (i.e. test) error. 
Intuitively, \ours's regularization incentivizes the underlying backbone's training (TransTEE) to focus on only a subset of the most important features, thereby reducing its variance and allowing it to perform better.

\begin{figure*}[!t]
  \centering
  \includegraphics[width=\textwidth]{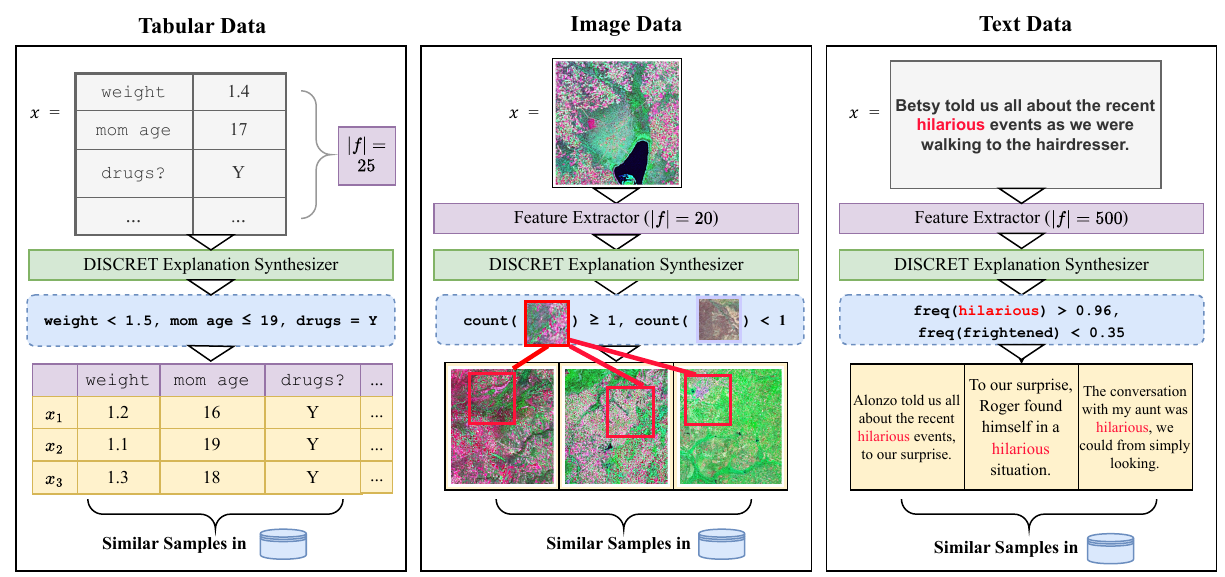} 
  \caption{\ours\ identifies similar samples across diverse datasets -- tabular (\ihdp), image (\uganda), and text (\eeec). 
  1) In the first setting, 
  given a tabular sample $x$ describing a premature infant, \ours\ establishes a rule associating extremely underweight ($\mathtt{weight} \le 1.5$) infants born to teenage mothers ($\mathtt{mom}\: \mathtt{age} \le 19$) with a history of drug use; such groups likely benefit from childcare visits (treatment), and will have highly improved cognitive outcomes. 2) In the second scenario on satellite images, 
  for a sample $x$, \ours\ discerns a rule based on the presence of concepts like "high soil moisture" (reddish-pink pixels) and absence of minimal soil (brown pixels); thus characterizing areas with high soil moisture. \ours's synthesized rule aligns with findings that government grants (treatment) are more effective in areas with higher soil moisture content (outcome) \citep{JJD-Heterogeneity}. 3) Likewise, the text setting aims to measure the impact of gender (treatment) on the mood (outcome). Given a sentence $x$ where the gendered noun ("Betsy") does not affect the semantic meaning, \ours's rule focuses on mood-linked words in the sentence, i.e., "hilarious".
}
  \label{fig:iclr_example}
\end{figure*}

\subsection{Misc. Experiments}
Appendix \ref{sec: addition_res} includes other experiments such as the ablation studies (Appendix \ref{sec: ablation}) with respect to dataset size and reward functions, and evaluating the training cost of \ours\ (Appendix \ref{sec:training_cost}).

\section{Related Work}

\textbf{Treatment effect estimation.} ML-based approaches to determine treatment effects can be divided into self-interpretable (often, tree-based), and deep-learning approaches. Deep-learning approaches mainly focus on how to appropriately incorporate treatment variables and covariates by designing various ad-hoc neural networks, such as Dragonnet \citep{shi2019adapting}, DRNet \citep{schwab2020learning} and TARNet \citep{shalit2017estimating}. Recently, it has been demonstrated that transformers \citep{zhang2022exploring} can encode covariates and treatment variables without any ad-hoc adaptations, which outperforms other deep-learning approaches. We thus select transformers as our default backbone models.

Self-interpretable models can be further subdivided into approaches specifically meant for causal inference, such as causal forests \cite{wager2018estimation}, and general-purpose models adapted to ITE such as random forests, Bayesian Additive Regression Trees (BART) \cite{bart2020}, ENRL \cite{shi2022explainable}. As shown earlier, these approaches are faithful, but often inaccurate. Prior work for treatment \textit{recommendation} has also used rules to drive model decisions \cite{lakkaraju17a}, but use static rule sets (rules and partitions of subgroups are pre-determined) and have been restricted to learning via Markov processes. In contrast, \ours\ enables dynamic rule generation for \textit{each sample} and predicts ITE accurately with deep reinforcement learning. 
Past approaches for treatment recommendation such as LEAP \cite{leap2017} have used reinforcement learning to fine-tune models, but were not inherently interpretable.

Recent work \citep{curth2024, chen2023covariate, nie2021quasi, kim2019learning} discusses key challenges in all ML-based solutions to ITE, notably interpretability and identifiability (i.e, ensuring the dataset contains appropriate features to infer treatment effects). Evidently, our work tackles interpretability by generating rule-based explanations. DISCRET enhances identifiability for image data via concept-extraction, in line with a suggestion by \cite{curth2024} to extract lower-dimensional information from the original feature space. 

\textbf{Model interpretability.} There are two lines of work to address the model interpretability issues, one is for interpreting black-box models in a post-hoc manner while the other one is for building a self-interpretable model. Post-hoc explainers could explain models with feature importance (e.g., Lime \citep{ribeiro2016should} and Shapley values \citep{shrikumar2017learning}) or logic rules (e.g., Lore \citep{guidotti2018local}, Anchor \citep{ribeiro2018anchors}). However, post-hoc explanations are usually not faithful \citep{rudin2019stop, bhalla2023verifiable}. To mitigate this issue, there are recent and ongoing efforts \citep{shi2022explainable,gautam2022protovae,huang2023one, you2023sum} in the literature to develop self-interpretable models.
For example, ENRL \citep{shi2022explainable} to learn tree-like decision rules and leverage them for predictions, ProtoVAE \citep{gautam2022protovae} learns prototypes and predicts the label of one test sample by employing its similarity to prototypes. 

\paragraph{Integrating rules into neural models.} 
How to integrate logic rules into neural models has been extensively studied \citep{seo2021controllingNN, seo2021controlling, khope2022critical, naik2023machine, naik2024torchql}. 
For instance, DeepCTRL \citep{seo2021controllingNN} has explored the use of \textit{existing} rules to improve the training of deep neural networks; in contrast, \ours\ does not require existing rules; it effectively learns (i.e. synthesizes) rules from training data and can be incorporated into neural models as regularization.

\paragraph{Program synthesis.} Program synthesis concerns synthesizing human-readable programs out of data, which has been extensively studied in the past few decades. Initial solutions, e.g., ILASP \cite{law2020ilasp} and Prosynth \cite{popl20} utilize pure symbolic reasoning to search logic rules. Recent approaches have explored neural-based synthesis, such as NeuralLP \cite{yang2017differentiable} and NLIL \cite{yang2019learn} for guiding the rule generation process.

\section{Conclusion \& Limitations}

In this work, we tackled the challenge of designing a faithful yet accurate AI model, \ours,
in the context of ITE estimation. 
To achieve this, we developed a novel deep reinforcement learning algorithm that is tailored to the task of synthesizing rule-based explanations. 
Extensive experiments across tabular, image, and text data demonstrate that \ours\ produces the most consistent (i.e. faithful) explanations, outperforms the self-interpretable models, is comparable in accuracy to black-box models, and can be combined with existing black-box models to achieve state-of-the-art accuracy. 

However, some limitations remain. \ours\ requires users to fix the grammar of explanations and set suitable hyper-parameters like the number of literals prior to training. Additionally, \ours\ relies on the extraction of interpretable symbols from unstructured data, like images. While the extraction of concepts from unstructured data is a widespread practice [8-10], DISCRET requires these concepts as input, which may not always be readily available. We leave these as avenues for future work.

\section*{Impact Statement}
Our work aims at the societally pertinent problem of Individual Treatment Estimation. A key positive impact of our work is improving trust in the faithfulness and explainability of ML predictions, especially in healthcare and poverty alleviation. In addition, we provide transparency to decision-makers who rely on treatment outcomes such as clinicians and policymakers. We do not foresee negative impacts of our work. As with all ML models, we caution end-users to rigorously test models for properties such as fairness (e.g. for implicit bias) before deploying them.

\section*{Acknowledgement}
This work is supported by ``The Fundamental Research Funds for the Central Universities, Peking University'', NSF Grant 2313010, and NIH Grants RF1AG063481 and U01CA274576. 

\bibliography{reference}
\bibliographystyle{icml2024}

\newpage

\onecolumn

\addcontentsline{toc}{section}{Appendices}
\setcounter{section}{0}
\renewcommand{\thesection}{\Alph{section}}

\section{Datasets}\label{sec: dataset}
\textbf{\ihdp} is a semi-synthetic dataset composed of the observations from 747 infants from the Infant Health
and Development Program, which is used for the effect of home visits (treatment variable) by specialists on infants' cognitive scores (outcome) in the future. 

\textbf{\tcga}. We obtain the covariates of TCGA from a real data set, the Cancer Genomic Atlas \citep{bica2020estimating}. We then follow the data generation process of \citep{zhang2022exploring} to generate synthetic treatments, dosage values and outcomes.

\textbf{\ihdpc} is a variant of the \ihdp\ dataset, where we modify the treatment variable to become continuous, and follow \citep{nie2020vcnet} to generate the synthetic treatment and outcome values. 

\textbf{\news} is composed of 3000 randomly sampled news items from the NY Times corpus \citep{misc_bag_of_words_164}. Bag-of-Word features are used for treatment effect estimation and we follow prior studies \citep{bica2020estimating} to generate synthetic treatment and outcome values. 

\textbf{\eeec} consists of 33738 English sentences. Each sentence in this dataset is produced by following a template such as ``$<$Person$>$ made me feel $<$emotional state word$>$'' where $<$Person$>$ and $<$emotional state word$>$ are placeholders to be filled. To study the effect of race or gender on the mood state, placeholders such as $<$Person$>$ are replaced with race-related or gender-related nouns (say an African-American name for $<$Person$>$) while the placeholder $<$emotional state word$>$ is filled with one of the four mood states: \textit{Anger}, \textit{Sadness}, \textit{Fear} and \textit{Joy}. The replacement of those placeholders with specific nouns is guided by a pre-specified causal graph \citep{feder2021causalm}. Throughout this paper, we only consider the case in which gender is the treatment variable. 

\textbf{\uganda} is composed of around 1.3K satellite images collected from around 300 different sites from \uganda. In addition to the image data, some tabular features are also collected such as age and ethnicity. However, as reported by \cite{jerzak2022image}, such tabular features often fail to cover important information such as the neighborhood-level features and geographical contexts, which, are critical factors for determining whether anti-poverty intervention for a specific area is needed. 

Note that the generation of synthetic treatments and outcomes on \ihdpc, \news\ and \tcga\ dataset relies on some hyper-parameters to specify the number of treatments or the range of dosage. For our experiments, we used the default hyper-parameters provided by \citep{zhang2022exploring}.


\section{Additional notes on baseline methods}\label{sec: add_baseline}
TVAE and Ganite can only handle binary treatments without dose variables, which are thus not applicable to \tcga, \ihdpc, and \news\ datasets. VCNet is not suitable for continuous treatment variables, and hence is not evaluated on \ihdpc\ and \news\ datasets.

\section{Additional Technical Details}\label{sec: addition_tech}

\subsection{Conventional Assumptions for Treatment Effect Estimation}\label{sec:te_assumption}

\begin{assumption}
(Strong Ignorability)
$Y (T = t) \perp T | X$. In the binary treatment case, $Y (0), Y (1) \perp T | X.$
\end{assumption}

\begin{assumption}
\vspace{-0.4em}
(Positivity)
$0 < \pi(T|X) < 1, \forall X, \forall T$.
\end{assumption} 
\begin{assumption}
\vspace{-0.4em}
(Consistency) For the binary treatment setting, $Y = T Y(1) + (1 - T) Y(0)$. 
\end{assumption} 


\subsection{Pre-training phase}\label{sec: pre-training phase}
As mentioned in Section \ref{sec: overview}, the backbone model $\Theta_0$ can be initialized with a pre-training phase. Specifically, we perform pre-training by training a black-box model, such as TransTEE, that leverages $\Theta_0$ as the encoder. We utilize the same training set during the pre-training phase as the one used during the training phase of \ours.    

\subsection{Encoding Rules}\label{sec: encoding_rules}
To encode a literal, $l_{k}=A\ op\ c$, we perform one-hot encoding on feature A and operator $op$, which are concatenated with the normalized version of $c$ (i.e., all the values of $A$ should be rescaled to $[0,1]$) as the encoding for $l_k$. We then concatenate the encoding of all $l_k$ to compose the encoding of $L_{1:K}$. 

\subsection{Generalizing to Disjunctive Rules}\label{sec: gen_disjunctive}
The above process of building a conjunctive rule can be viewed as generating {\it the most probable} conjunctive rules among all the possible combinations of $A$, ${op}$ and $c$. This can be generalized to building a rule with multiple disjunctions, by generating the $H$ most probable conjunctive rules instead, where $H$ represents the number of disjunctions specified by users. Specifically, for the model $\Theta_1$, we simply select the $H$ most probable features from its model output while for the model $\Theta_2$, we leverage beam search to choose the $H$ most probable $(A, c)$ pairs. 

\subsection{Generalizing to Categorical Outcome Variables}
To generalize \ours\ to handle categorical outcome variables, by following \citep{feder2021causalm}, the treatment effect is defined by the difference between the probability distributions of all categorical variables. Additionally, to estimate outcomes within a subgroup of similar samples, we simply compute the frequency of each outcome as the estimation. 

\subsection{Generalizing to Other Categories of Treatment Variables}\label{sec: general_treatment_var}
We first discuss general settings for various treatment variables and then discuss how to estimate the treatment effect for each of them.

The settings for all treatment variables that our methods can deal with:
\begin{enumerate}[leftmargin=*,nolistsep, nosep]
    \item Tabular data with a binary treatment variable $T$ and no dose variables. In this setting, $T=1$ represents treated unit while $T=0$ represents untreated unit, and the ITE is defined as the difference of outcomes under the treatment and under the control, respectively (i.e., $\text{ITE}(x) = y_1(x)-y_0(x)$, where $y_1(x)$ and $y_0(x)$ represents the potential outcome with and without receiving treatment for a sample $x$).  
    The average treatment effect, $\text{ATE}$, is the sample average of $\text{ITE}$ across all samples (i.e., $\text{ATE} = \mathbb{E}[\text{ITE}]$).
    
    \item Tabular data with a continuous treatment variable $T$. Following \cite{zhang2022exploring}, the average dose-response function is defined as the treatment effect, i.e., $\mathbb{E} [Y|X, do(T=t)]$.
    \item Tabular data with a discrete treatment variable $T$ with one additional continuous dose variable $S$. Following \cite{zhang2022exploring}, the average treatment effect is defined as the average dose-response function: $\mathbb{E} [Y|X, do(T=t, S=s)]$.

\end{enumerate}

The treatment effect for each of the above settings is then estimated as follows:
\begin{enumerate}[leftmargin=*,nolistsep, nosep]
    \item With a binary treatment variable and no dose variable, 
    we can estimate the ATE of $R_x(\mathcal{D})$ via arbitrary treatment effect estimation methods, such as the classical statistical matching algorithm \citep{kline2022psmpy}, or state-of-the-art neural network models. In this paper, we adopt the K-Nearest Neighbor Matching by default for estimating the ATE of $R_x(\mathcal{D})$:
$\text{ITE} = y_1(x) - y_0(x)$.
    We can also obtain the estimated outcome by averaging the outcome of samples from $R_x(\mathcal{D})$ with the same treatment as the sample $x$, i.e.:
    \begin{small}
    \begin{align}
        \begin{split}
            \widehat{y}(t) = \frac{1}{\sum \mathbb{I}({\othert{i}=t})}\sum\mathbb{I}({\othert{i}=t})\cdot \othery{i}
        \end{split}
    \end{align}    
    \end{small}
    \item With a continuous treatment variable $T$ but without dose variables, then as per \ref{sec: prelim}, the ITE is represented by the outcome conditioned on the observed treatment. One straightforward way to estimate it is to employ the average outcome of samples within $R_x(\mathcal{D})$ that receive similar treatments to $x$, which is also the estimated outcome for this sample:
    \begin{small}
        \begin{align}\label{eq: ite_cont_treatment}
            \begin{split}
                \hat{y} 
                = \frac{\sum\mathbb{I}[{(\otherx{i}, \othert{i}, \othery{i}) \in \text{top}_k(R_x(\mathcal{D}))}]\cdot\othery{i}}{\sum\mathbb{I}[{(\otherx{i}, \othert{i}, \othery{i}) \in \text{top}_k(R_x(\mathcal{D}))}]},
            \end{split}
        \end{align}
    \end{small}
    in which $\text{top}_k(R_x(\mathcal{D}))$ is constructed by finding the top-$k$ samples from $R_x(\mathcal{D})$ with the most similar treatments to $x$. But again, any existing treatment effect estimation methods for continuous treatment variables from the literature are applicable to estimate $\widehat{\text{ITE}}_x$.

    \item With a discrete treatment variable $T$ and one associated continuous dose variable $S$, ITE is estimated in a similar way to \eqref{eq: ite_cont_treatment}. Specifically, we estimate ATE over the subgroup of similar samples with the following formula:
    \begin{small}
        \begin{align}\label{eq: ite_treatment_dose}
            \begin{split}
                \hat{y}
                = \frac{\sum\mathbb{I}[{(\otherx{i}, \othert{i}, \others{i}, \othery{i}) \in \text{top}_k(R_x(\mathcal{D}))}]\cdot\othery{i}}{\sum\mathbb{I}[{(\otherx{i}, \othert{i}, \others{i}, \othery{i}) \in \text{top}_k(R_x(\mathcal{D}))}]}.
            \end{split}
        \end{align}
    \end{small}
    In the above formula, 
    $\text{top}_k(R_x(\mathcal{D}))$ is constructed by first selecting the samples with the same treatment as the sample $x$ and then only retaining the $k$ samples with the most similar dose values to $x$.
\end{enumerate}


\subsection{Deep Q-learning and Training Algorithm}
To facilitate Q-learning, we estimate the Q value 
with the output logits of the models given a state $(x, L_{1:k-1})$ and an action $l_k$, which is denoted by $Q(l_k, (x, L_{1:k-1}))$. Note that $l_k$ is generated collaboratively by using two models, $\Theta_1$ and $\Theta_2$, we therefore need to collect two sub-Q values from these two models, and 
then aggregate (say average) them as the overall Q value, which follows prior multi-agent Q-learning literature \citep{wang2021towards}. 
In the end, by following the classical DQL framework, we optimize the following objective function adapted from the Bellman equation \citep{dixit1990optimization}:
\begin{small}
    \begin{equation}
    L_{\Theta} = \mathbb{E}[Q(l_k, (x, L_{1:k-1})) - (\gamma\cdot \max_{l_{k+1}}Q(l_{k+1}, (x, L_{1:k})) + r_k) ]^2,
    \end{equation}
\end{small}
which is estimated over a sampled mini-batch of cached experience taking the form of $<(x,L_{1:k-1}), l_k, r_k, (x,L_{1:k})>$ during the experience replay process.
The training algorithm for rule learning is outlined in Algorithm \ref{alg:training_overview} below. 

\begin{algorithm}[H]
\begin{small}
\caption{The overview of Deep Q-Learning (DQL) algorithm for rule learning in \ours}
\label{alg:training_overview}
\textbf{Input}: {target model update: $t$, gamma: $\gamma$, batch size: $b$, target model parameters: $\Theta^{target}$, policy model parameters: $\Theta^{policy}$, experience replay cache: $ cache = <(x,L_{1:k-1}), l_k, r_k, (x,L_{1:k})>$ where $x$ is a covariate, $L_{1:k-1}$ is the set of literals at step $k-1$, $l_k$ is the literal synthesized at step $k$, $r_k$ is the reward at step $k$, and $L_{1:k}$ is $L_{1:k-1} \cup l_k$}\\
\textbf{Output}: {$None$}
\begin{algorithmic}[1]
    \STATE Initialize $\displaystyle \vw^{pred}$ and $\displaystyle \vw^{target}$ of length $b$
    \STATE Construct $batch$ by sampling $b$ entries from $cache$
    \FOR{$i, <(x^i,L_{1:k-1}^i), l_k^i, r_k^i, (x^i,L_{1:k}^i)>$ in $\emph{Enumerate}(batch)$}
        \STATE Use $\Theta_0^{policy}$ and a deterministic function to encode both $x^i$ and $L_{1:k-1}^i$, respectively, to get $E_{k-1}^i$;
        \STATE Forward pass $E_{k-1}^i$ through $\Theta_1^{policy}$ and select the index of the feature from $l_k^i$ to obtain $Q_f^i$;
        \STATE Append a one-hot encoding of the feature from $l_k^i$ to $E_{k-1}^i$ to get $E_{partial}^i$;
        \STATE forward pass $E_{partial}^i$ through $\Theta_2^{policy}$ and select the index of the constant from $l_k^i$ to get $Q_c^i$;
        \STATE Obtain $Q_{k-1}^i$ by averaging $Q_f^i$ and $Q_c^i$;
        \STATE Obtain $Q_k^i$ by forward passing $x^i$ and $L_{1:k}^i$ through $\Theta^{target}$ and averaging the maximum $Q$ values from $\Theta_1^{target}$ and $\Theta_2^{target}$;
        \STATE $\displaystyle \vw^{pred}_i \gets Q_{k-1}^i$; $\displaystyle \vw^{target}_i \gets \gamma Q_k^i * + r_k^i$;
    \ENDFOR
    \STATE Backpropogate and update $\Theta^{policy}$ using loss $\emph{MSE}(\displaystyle \vw^{pred}$,$\displaystyle \vw^{target})$
    \IF{$\emph{len}(cache) \% t == 0$}
    \STATE $\Theta^{target} \gets \Theta^{policy}$
    \ENDIF
\end{algorithmic}
\end{small}
\end{algorithm}

\subsection{Proof of Theorem \ref{thm: main}}

We first state some additional preliminary notations and settings for Q-learning.  We denote the Markov decision process (MDP) as a tuple $(\mathcal{S}_k, \mathcal{L}_k, P_k, r_k)$ where

\begin{itemize}
    \item $\mathcal{S}_k$ is the state space with a state $(x, L_{1:k})$;
    \item $\mathcal{L}_k$ is the action space with an action $l_k$;  
    \item $P_k$ represents the transition probability;
    \item $r_k$ represents the reward function. 
\end{itemize}

 Theorem \ref{thm: main} is a direct implication of Lemma \ref{lem: y} below. 

\begin{lemma} \label{lem: y}
Suppose we have input data $\{(x_i, t_i, s_i, y_i)\}_{i=1}^N$ where $x_i \in \mathbb{R}^m$ and discrete, $t_i \in \mathbb{R}, s_i \in \mathbb{R}$ and $y_i \in \mathbb{R}$, then the $\hat y$ obtained from DISCRET converges to zero generalization error with probability 1 (i.e. $(y - \hat y)^2 \rightarrow 0$ w.p. 1) for any fixed $K \leq m$ over the dataset with all discrete features under the data generating process $y = f (\mathcal{X}_K ) + c \cdot t + \epsilon$, where $\mathcal{X}_K \subseteq \{X_1, X_2, \cdots, X_m \}, c \in \mathbb{R}, t$ is the treatment assignment, and $\epsilon \sim \mathcal{N}(0, \sigma^2)$ for some $\sigma > 0$. 
\end{lemma}

To prove Lemma \ref{lem: y}, we need to use results from \ref{thm: Q}. 

\begin{theorem} \label{thm: Q}
    Given a finte Markov decision process $(\mathcal{S}_k, \mathcal{L}_k, P_k, r_k)$, given by the update rule
\begin{equation} \label{equ: Bellman}
    \begin{split}
                Q(l_k, (x, L_{1:k})) &= Q(l_{k-1}, (x, L_{1:k-1})) + \alpha_{k-1} (l_{k-1}, (x, L_{1:k-1})) \times \\
        &\left(r_{k-1} + \gamma \max_{(x^*, L^*_{1:k-1}) \in \mathcal{S}_k \times \mathcal{L}_k}Q (l_{k-1}, (x^*, L^*_{1:k-1})) - Q (l_{k-1}, (x, L_{1:k-1})) \right) 
    \end{split}
\end{equation}
    converges with probability 1 to the optimal $Q$-function as long as 
    \begin{align*}
        \sum_k \alpha_k (l_k, (x, L_{1:k-1})) = \infty, \quad \sum_k \alpha_k^2 (l_k, (x, L_{1:k-1})) < \infty
    \end{align*}
    for all $(l_k, (x, L_{1:k-1})) \in \mathcal{S}_k \times \mathcal{L}_k$. 
\end{theorem}

\begin{proof}
    We start rewriting equation (\ref{equ: Bellman}) as
\begin{align*}
     Q(l_k, (x, L_{1:k})) &= \left(1 - \alpha_{k-1} (l_{k-1}, (x, L_{1:k-1})) \right) Q(l_{k-1}, (x, L_{1:k-1})) +  \alpha_{k-1} (l_{k-1}, (x, L_{1:k-1})) \times \\
        &\left(r_{k-1} + \gamma \max_{(x^*, L^*_{1:k-1}) \in \mathcal{S}_k \times \mathcal{L}_k}Q (l_{k-1}, (x^*, L^*_{1:k-1})) \right) 
\end{align*}
Denote the optimal $Q$ function be $Q^*(l_k, (x, L_{1:k})) $, subtracting equation above from both sides the quantity $Q^*(l_k, (x, L_{1:k})) $ and letting 
\begin{align*}
    \Delta_k (l_k, (x, L_{1:k})  ) = Q(l_k, (x, L_{1:k})) - Q^*(l_k, (x, L_{1:k})) 
\end{align*}
yields 
    \begin{align*}
            \Delta_k (l_k, (x, L_{1:k})  ) &= \left(1 - \alpha_{k-1} (l_{k-1}, (x, L_{1:k-1})) \right) \Delta_k (l_k, (x, L_{1:k})  ) \\
            &+ \alpha_{k-1} (l_{k-1}, (x, L_{1:k-1}))\left(r_k + \gamma  \max_{(x^*, L^*_{1:k-1}) \in \mathcal{S}_k \times \mathcal{L}_k}Q (l_{k-1}, (x^*, L^*_{1:k-1})) -  Q^*(l_k, (x, L_{1:k}))  \right).
    \end{align*}
If we write
\begin{align*}
    F_k (l_k, (x, L_{1:k})   ) = r_k((x, L_{1:k}), l_k, \mathcal{S}(x, L_{1:k}) ) + \gamma \max_{(x^*, L^*_{1:k-1}) \in \mathcal{S}_k \times \mathcal{L}_k}Q (l_{k-1}, (x^*, L^*_{1:k-1})) -  Q^*(l_k, (x, L_{1:k}))
\end{align*}
where $ \mathcal{S}(x, L_{1:k})$ is a random sample state obtained from the Markov chain $(\mathcal{S}_k, P_k)$, we have
\begin{align*}
    & \qquad \mathbb{E}[F_k (l_k, (x, L_{1:k})   ) | \mathcal{F}_k ] \\
    &= \sum_{b \in \mathcal{S}_k } P_k ((l_k, (x, L_{1:k}), b ) [r_k ((l_k, (x, L_{1:k}), l_{k} ) + \gamma  \max_{(x^*, L^*_{1:k-1}) \in \mathcal{S}_k \times \mathcal{L}_k}Q (l_{k-1}, (x^*, L^*_{1:k-1})) -  Q^*(l_k, (x, L_{1:k})) ] \\
    & = (\mathbf{H} Q) (x, L_{1:k})  - Q^*(l_k, (x, L_{1:k})).
\end{align*}
Using the fact that $Q^* =  (\mathbf{H} Q) (x, L_{1:k})$,
\begin{align*}
     \mathbb{E}[F_k (l_k, (x, L_{1:k})   ) | \mathcal{F}_k ] =  (\mathbf{H} Q) (x, L_{1:k} -  (\mathbf{H} Q^*) (x, L_{1:k} \leq \gamma \|Q - Q^*  \| = \gamma \| \Delta_k \|_\infty.
\end{align*}
We could also verify that
\begin{align*}
    Var [F_k (l_k, (x, L_{1:k})   ) | \mathcal{F}_k] \leq C (1 + \|\Delta_k \|_W^2)
\end{align*}
for some constant $C$. Then by the theorem below, $\Delta_k$ converges to zero with probability 1. Hence, $Q$ converges to $Q^*$ with probability 1.
\end{proof}

\begin{theorem}[\citet{jaakkola1993convergence}] 
    The random process $\{\Delta_t \}$ taking values in $\mathbb{R}^n$ and defined as 
    \begin{align*}
        \Delta_{t+1} (x) = (1 - \alpha_t (x)) \Delta_t (x) + \alpha_t (x) F_t (x)
    \end{align*}
    converges to zero with probability 1 under the following assumptions:
    \begin{itemize}
        \item $0 \leq \alpha_t \leq 1, \sum_{t} \alpha_t (x) = \infty$ and $\sum_t \alpha_t^2(x) < \infty$;
        \item $\|\mathbb{E} [F_t (x) | \mathcal{F}_t] \|_W \leq \gamma \|\Delta_t \|_W,$ with $\gamma < 1$;
        \item $Var(\mathcal{F}_t(x)|\mathcal{F}_t) \leq C(1 + \|\Delta_t \|_W^2),$ for $C > 0$.
    \end{itemize}
\end{theorem}

\begin{proof}
    See \citet{jaakkola1993convergence} for the proof. 
\end{proof}

\begin{proof}[Proof of Lemma \ref{lem: y}]
Using Theorem \ref{thm: Q}, we can see that $Q$ obtained from DISCRET converges to optimal $Q^*$. As a result, $hat y$ obtained from DISCRET converges to optimal $y^*$. We left to prove that $y^*$ leads to a zero mean square error (i.e., $\|y - y^* \|_2^2$).  We can prove this using the fact that all features are discrete. Since all features are discrete and the optimal feature being selected in each step leads to a zero mean square error and other features lead to non-zero mean square error, it turns out that $y^*$ obtained from DISCRET leads to a zero mean square error.
\end{proof}

\subsection{Additional Reward Function Optimizations}\label{sec: reward_func_additional}

We further present some strategies to optimize the design of the cumulative reward function defined in \eqref{eq: reward_function}, which includes incorporating estimated propensity scores into this formula and automatically fine-tuning its hyper-parameters.

\paragraph{Regularization by estimating propensity scores} Similar to prior studies on ITE estimation \citep{shi2019adapting, zhang2022exploring}, we regularize the reward function $r_k$ by integrating the estimated propensity score, $\widehat{\pi}(T=t| X=x)$. Specifically, for discrete treatment variables, we re-weight \eqref{eq: reward_function_2} with the propensity score as a regularized reward function, i.e.:

\vspace{-2em}
\begin{small}
    \begin{align}\label{eq: reward_function_2}
        \begin{split}
        V_{1:K}^{reg} & =[e^{-\alpha(y - \widehat{y}_{1:K})^2} + \beta \cdot \widehat{\pi_{1:k}}(T=t| X=x)]\\
        & \cdot \mathbb{I}(L_{1:K}(\mathcal{D})\text{ is non-empty}),  
        \end{split}
    \end{align} \par
\end{small}
\vspace{-1.1em}

\paragraph{Automatic hyper-parameter fine-tuning}

We further studied how to 
automatically tune the hyper-parameter $\alpha$ and $\beta$ in \eqref{eq: reward_function_2}. For $\alpha$, 
at each training epoch, we identify the training sample producing the median of $(y-\widehat{y}_{1:K})^2$ among the whole training set and then ensure that for this sample, \eqref{eq: reward_function} is 0.5 through adjusting $\alpha$. This can guarantee that for those training samples with the smallest or largest outcome errors, \eqref{eq: reward_function} approaches 1 or 0 respectively.

We also designed an annealing strategy to dynamically adjust $\beta$ by setting it as 1 during the initial training phase to focus more on  
treatment predictions, and switching it to 0 
so that reducing outcome error is prioritized in the subsequent training phase.

\section{Addendum on Performance Metrics}\label{sec: metrics}

\subsection{Faithfulness Metrics}
We evaluate the faithfulness of explanations with two metrics, i.e., consistency and sufficiency from \citep{dasgupta2022framework}. For a single sample $x$ with local explanation $e_x$, the consistency is defined as the probability of getting the same model predictions for the set of samples producing the same explanations (denoted by $C_x$) as $x$ while the sufficiency is defined in the same way, except that it depends on the set of samples satisfying $e_x$ (denoted by $S_x$) rather than generating explanation $e_x$. These two metrics could be formalized with the following formulas:
\begin{align*}
    & \text{Consistency}(x) = Pr_{x' \in_{\mu}C_x} (\hat{y}(x) == \hat{y}(x')) \\
    & \text{Sufficiency}(x) = Pr_{x' \in_{\mu}S_x} (\hat{y}(x) == \hat{y}(x'))
\end{align*}

in which $\mu$ represents the probability distribution of $C_x$ and $S_x$. To evaluate explanations with these two metrics, \citep{dasgupta2022framework} proposed an unbiased estimator for $\text{Consistency}(x)$ and $\text{Sufficiency}(x)$, i.e.,:
\begin{align*}
    & \widehat{\text{Consistency}}(x) = \frac{1}{N} \sum_{i=1}^N \mathbb{I}(C_{x} > 1)\cdot\frac{C_{x, \hat{y}(x)}-1}{C_{x}-1} \\
    & \widehat{\text{Sufficiency}}(x) = \frac{1}{N} \sum_{i=1}^N \mathbb{I}(S_{x} > 1)\cdot\frac{S_{x, \hat{y}(x)}-1}{S_{x}-1}
\end{align*}

in which $C_{x, \hat{y}(x)}$ represents the set of samples sharing the same explanation and the same model predictions as the sample $x$ while $S_{x, \hat{y}(x)}$ represents the set of samples that satisfy the explanation produced by $x$ and share the same explanation as $x$. As the above formula suggests, both the consistency and sufficiency scores vary between 0 and 1.

But note that for typical ITE settings, the model output is continuous rather than discrete numbers. Therefore, we discretize the range of model output into evenly distributed buckets, and the model outputs that fall into the same buckets are regarded as having the same model predictions. As \citep{dasgupta2022framework} mentions, the sufficiency metric is a reasonable metric for evaluating rule-based explanations since it requires retrieving other samples with explanations. So we only report sufficiency metrics for methods that can produce rule-based explanations in Table \ref{tab: tabular_sufficiency}.

\subsection{Additional Notes for the \eeec\ Dataset} Note that for \eeec\ dataset, $\epsilon_{ATE}$ is used for performance evaluation but the ground-truth ITE is not observed, which is approximated by the difference of the predicted outcomes between factual samples and its ground-truth counterfactual alternative \citep{feder2021causalm}.

\subsection{$AMSE$ for Continuous Treatment Variable or Dose Variable}
To evaluate the performance of settings with continuous treatment variables or continuous dose variables, we follow \citep{zhang2022exploring} to leverage $AMSE$ as the evaluation metrics, which is formalized as follows:
\begin{align*}
AMSE = \begin{cases}
    \frac{1}{N}\sum_{i=1}^N \int_{t}[\hat{y}(x_i,t) - y(x_i,t)]\pi(t)dt & \text{continuous treatment variable} \\
    \frac{1}{N T}\sum_{i=1}^N \sum_{t=1}^T \int_{s}[\hat{y}(x_i,t) - y(x_i,t)]\pi(t)dt & \text{continuous dose variable},
\end{cases}    
\end{align*}

in which we compute the difference between the estimated outcome $\hat{y}$ and the observed outcome $y$ conditioned on every treatment $t$, and average this over the entire treatment space and all samples for evaluations. Due to the large space of exploring all possible continuous treatments $t$ or continuous dose values $s$, we collect sampled treatment or sampled dose rather than enumerate all $s$ and $t$ for the evaluations of $AMSE$.

\section{Additional Experimental Results}\label{sec: addition_res}

\subsection{Performance of Self-interpretable Models with Varying Complexity}\label{sec: varied_complexity}

On evaluating the performance of self-interpretable models when trained with a high depth, i.e number of conjunctive clauses ($K = 100$, as opposed to low-depth $K = 6$, see Table \ref{tab: self-interpret-high-depth}), we see that \ours\ ($K = 6$) outperforms these models despite having lower depth, and thus better interpretability. 

\begin{table*}[h]
\small
\centering
\setlength{\tabcolsep}{4pt}
\begin{tabular}{c|c|c|cc|cc|c}\toprule
\textbf{Modality} $\rightarrow$ & & & \multicolumn{5}{c}{\textbf{Tabular}} \\ \midrule
\textbf{Dataset} $\rightarrow$  & & & \multicolumn{2}{c|}{\ihdp} & \multicolumn{2}{c|}{\tcga} & \multicolumn{1}{c}{\ihdpc} \\ \midrule
\makecell{\textbf{Method} $\downarrow$} & 
\makecell{Trees} & 
\makecell{Depth} & 
\makecell{$\epsilon_{\text{ATE}}$\\(In-sample)} & \makecell{$\epsilon_{\text{ATE}}$ \\ (Out-of-sample)} &\makecell{$\epsilon_{\text{ATE}}$\\(In-sample)} &\makecell{$\epsilon_{\text{ATE}}$\\ (Out-of-sample)} & AMSE \\ \hline  

\multirow{2}{*}{Decision Tree} & 
  - & 6 & 0.693$\pm$0.028&0.613$\pm$0.045 & 0.200$\pm$0.012&0.202$\pm$0.012 &21.773$\pm$0.190 \\ &
%
 - & 100 & 0.638$\pm$0.031 & 0.549$\pm$0.052 & 0.441$\pm$0.004 & 0.445$\pm$0.004& 23.382$\pm$0.342 \\ 
\hline

\multirow{3}{*}{Random Forest} & 
  1 & 6 & 0.801$\pm$0.039 &0.666$\pm$0.055 & 19.214$\pm$0.163 &19.195$\pm$0.163 &21.576$\pm$0.185 \\ &
%
  1 & 100 & 0.734$\pm$0.041	& 0.653$\pm$0.056	 & 0.536$\pm$0.011	 & 0.538$\pm$0.012	 & 33.285$\pm$0.940 \\ &
  10 & 100 & 0.684$\pm$0.033	 & 0.676$\pm$0.034	 & 0.536$\pm$0.011	 & 0.538$\pm$0.012 & 38.299$\pm$0.841\\ 
\hline

NAM & 
 - & - & 0.260$\pm$0.031&0.250$\pm$0.032 &- & - &24.706$\pm$0.756 \\ 
\hline

\multirow{2}{*}{ENRL} & 
 1 & 6 & 4.104$\pm$1.060& 3.759$\pm$0.087&10.938$\pm$2.019 &10.942$\pm$2.019&24.720$\pm$0.985 \\ &
%
 1 & 100 & 4.094$\pm$0.032 & 4.099$\pm$0.107 & 10.938$\pm$2.019 & 10.942$\pm$2.019 & 24.900 $\pm$ 0.470 \\
\hline

\multirow{4}{*}{Causal Forest} & 
 1 & 6 & 0.144$\pm$0.019&0.275$\pm$0.035 & - & - & - \\ &
 1 & 100 & 0.151$\pm$0.019 & 0.278$\pm$0.033 & - & - & - \\ &
 100 & max & 0.124$\pm$0.015	& 0.230$\pm$0.031 & - & - & - \\ 
\hline

\multirow{2}{*}{BART} & 
 1 & - & 1.335$\pm$0.159&1.132$\pm$0.125&230.74$\pm$0.312&236.81$\pm$0.531&12.063$\pm$0.410 \\ &
 $N$ & - & 0.232$\pm$0.039 & 0.284$\pm$0.036 & - & - & 4.323$\pm$0.342 \\ 
\hline

\ours\ (ours) & - & 6 & \underline{0.089$\pm$0.040} & \underline{0.150$\pm$0.034} & \underline{0.076$\pm$0.019} & \underline{0.098$\pm$0.007} & \underline{0.801$\pm$0.165} \\ 
\makecell{TransTEE + \ours\\ (ours)*} & - & - & \textbf{0.082$\pm$0.009} &\textbf{0.120$\pm$0.014} & \textbf{0.058$\pm$0.010}&\textbf{0.055$\pm$0.009}&\textbf{0.102$\pm$0.007} \\
\bottomrule 
\hline
\end{tabular}
\caption{
ITE estimation errors (lower is better) at varying complexities for self-interpretable models. We \textbf{bold} the smallest estimation error for each dataset,  and \underline{underline} the second smallest one. Results in the first row for each method are duplicated from Table \ref{tab: quantative_res}. For BART, we set $N = 200$ for IHDP, and $N = 10$  for TCGA and IHDP-C due to large feature number of features in the latter. We show that \ours\ outperforms self-interpretable models and has simpler rules regardless of the model complexity used. Asterisk (*) indicates model is not self-interpretable. 
}
\label{tab: self-interpret-high-depth}
\vspace{-10px}
\end{table*}

It is worth noting that in both Table \ref{tab: quantative_res} and Table \ref{tab: self-interpret-high-depth}, the ITE errors for the \ihdpc\ dataset are pretty high for the baseline self-interpretable models and some black box models. This is because computing ITE for the \ihdpc\ dataset is a particularly hard problem, and necessitates the use of powerful models with high complexity. Indeed, \ihdpc\ dataset is a semi-synthetic dataset where values of the outcome variable are generated by a very complicated non-linear function \citep{zhang2022exploring}. Hence, tree-based models may not be able to capture such complicated relationships well. This is evidenced by high training errors and likely underfitting (training error was 48.17 for random forest v/s 0.58 for \ours). Even simple neural networks such as TARNet and DRNet, also significantly underperform as Table \ref{tab: quantative_res} suggests. Thus, ITE for \ihdpc\ can only be effectively encoded by powerful models, such as \ours\ and transformer-based architectures like TransTEE.

\subsection{Ablation Studies}\label{sec: ablation}
We further perform ablation studies to explore how different components of \ours\, such as the database and featurziation process (for NLP and image data), affect the ITE estimation performance. In what follows, we analyze the effect of the size of the database, different featurization steps, and different components of the reward function.

\textbf{Ablating the reward functions for \ours.}
Recall that in Section \ref{sec: training}, the reward function used for the training phase could be enhanced by adding propensity scores as one regularization and automatically tuning the hyper-parameters, $\alpha$ and $\beta$. We removed these two components from the reward function one after the other to investigate their effect on the ITE estimation performance. We perform this experiment on \uganda\ dataset and report the results in Table \ref{tab:ablation_reward}. As this table suggests, throwing away those two components from the reward function incurs higher outcome errors, thus justifying the necessity of including them for more accurate ITE estimation. 

\begin{table}[ht]
\small
    \centering
    \begin{tabular}{cc} \toprule
    & Outcome error \\\hline
        \ours & \textbf{1.662$\pm$0.136} \\ \midrule
        \ours\ without propensity score & 1.701$\pm$0.161\\ \midrule
        \ours\ without propensity score or auto-finetuning &1.742$\pm$0.151 \\ \bottomrule
    \end{tabular}
    \caption{Ablation studies on the reward function in \ours}
    \label{tab:ablation_reward}
\end{table}

\textbf{Ablating the database size.}
Since \ours\ estimates ITE through rule evaluations over a database, the size of this database can thus influence the estimation accuracy. We therefore vary the size of the IHDP dataset, i.e., the number of training samples, and compare \ours\ against baselines with varying database size. The full results are included in Table \ref{tab:ablation_dataset}.
As expected, error drops with increasing dataset size, and DISCRET outperforms baselines (particularly self-interpretable models) at smaller dataset sizes. The results suggest that with varied dataset sizes, TransTEE + DISCRET still outperforms all baseline methods while DISCRET performs better than all self-interpretable models. It is also worth noting that when the database size is reduced below certain level, e.g., smaller than 200, \ours\ can even outperform TransTEE. This implies that \ours\ could be more data-efficient than the state-of-the-art neural network models for ITE estimations, which is left for future work.

\begin{table}[ht]
\small
\centering
\begin{tabular}{c|c|c|c|c}
\toprule
\textbf{Method} & \textbf{100} & \textbf{200} & \textbf{400} & \textbf{Full (747)}  \\
\midrule
Decision Tree 	& 7.08$\pm$4.61 	& 1.04$\pm$0.30 & 1.19$\pm$0.52 & 	0.73$\pm$0.13 \\
Random Forest & 8.05$\pm$5.15 & 1.43$\pm$0.39 & 0.63$\pm$0.19 & 0.87$\pm$0.12 \\
NAM & 1.56$\pm$0.86 & 0.46$\pm$0.21 & 0.75$\pm$0.46 & 0.29$\pm$0.13 \\
ENRL & 4.40$\pm$0.33 & 4.05$\pm$0.04 & 4.40$\pm$0.33 & 4.05$\pm$0.05 \\
Causal forest & 0.87$\pm$0.47 & 0.88$\pm$0.24 & 0.31$\pm$0.14 & 0.18$\pm$0.06 \\
BART & 3.32$\pm$0.71 & 1.54$\pm$0.59 & 1.46$\pm$0.80 & 0.71$\pm$0.22 \\
DISCRET & 0.55$\pm$0.13 & 0.47$\pm$0.10 & 0.32$\pm$0.15 & 0.21$\pm$0.05 \\
Dragonnet & 0.94$\pm$0.47 & 0.46$\pm$0.09 & 1.06$\pm$0.61 & 0.23$\pm$0.08 \\
TVAE & 4.35$\pm$0.33 & 4.00$\pm$0.04 & 4.35$\pm$0.33 & 3.87$\pm$0.05 \\
TARNet & 0.33$\pm$0.12 & \underline{0.23$\pm$0.03} & \underline{0.16$\pm$0.03} & 0.17$\pm$0.03 \\
Ganite & 0.65$\pm$0.23 & 0.32$\pm$0.04 & 0.75$\pm$0.26 & 0.57$\pm$0.11 \\
DRNet & 0.37$\pm$0.11 & 0.43$\pm$0.23 & 0.19$\pm$0.06 & 0.17$\pm$0.03 \\
VCNet & 4.27$\pm$0.29 & 3.98$\pm$0.04 & 4.09$\pm$0.31 & 3.95$\pm$0.06 \\
TransTEE & \underline{0.33$\pm$0.05} & 0.35$\pm$0.15 & \underline{0.16$\pm$0.07} & \underline{0.15$\pm$0.03} \\
DISCRET & 0.55$\pm$0.13 & 0.47$\pm$0.10 & 0.32$\pm$0.15 & 0.21$\pm$0.05 \\
TransTEE + DISCRET & \textbf{0.24$\pm$0.05} & \textbf{0.21$\pm$0.06} & \textbf{0.09$\pm$0.03} & \textbf{0.08$\pm$0.03} \\
\bottomrule
\end{tabular}
\caption{ITE test errors (out-of-sample) with varied numbers of samples randomly selected from IHDP dataset}
\label{tab:ablation_dataset}
\end{table}


\subsection{Training Cost of DISCRET}\label{sec:training_cost}
We further plot Figure \ref{fig:training_cost} to visually keep track of how the ATE errors on test set are evolved throughout the training process. As this figure suggests although the best test performance occurs after 200 epochs (ATE error is around 0.12). However, the performance in the first few epochs is already near-optimal (ATE error is around 0.14). Therefore, despite the slow convergence in typical reinforcement learning training processes, our methods obtain reasonable treatment effect estimation performance without taking too many epochs. 

\begin{figure}[H]
    \centering
    \includegraphics[width=0.6\textwidth]{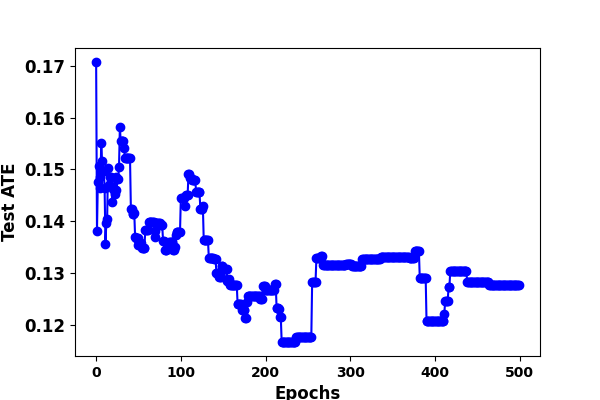}
    \caption{The curve of ATE errors on test split of \ihdp\ by \ours}
    \label{fig:training_cost}
\end{figure}


\subsection{Consistency and Sufficiency Scores} \label{sec:consistency_scores}
We provide the full results of the consistency and sufficiency scores below.

\begin{table}[ht]
\small
\centering
\begin{tabular}{ccccccc}\toprule
& \ihdp\ & \tcga\ &\ihdpc\ &  \news\ & \eeec\ & \uganda\ \\ \hline 
Model distillation &0.243$\pm$0.126 & 0.562$\pm$0.026 &0.127$\pm$0.008 &0.816$\pm$0.032 &0.004$\pm$0.001 & 0.198$\pm$0.008\\
Lore &0.000$\pm$0.000 &0.000$\pm$0.000& 0.000$\pm$0.000&0.000$\pm$0.000&0.000$\pm$0.000 & 0.000$\pm$0.001\\
Anchor &0.084$\pm$0.083 &0.001$\pm$0.000 &0.293$\pm$0.022 &0.000$\pm$0.000&0.000$\pm$0.000&0.066$\pm$0.015\\
Lime &0.182$\pm$0.129 &0.000$\pm$0.000&0.001$\pm$0.001 &0.000$\pm$0.000&0.000$\pm$0.000&0.000$\pm$0.000\\
Shapley & 0.009$\pm$0.017&0.005$\pm$0.002 &0.046$\pm$0.027 & 0.031$\pm$0.035&0.034$\pm$0.003 & 0.412$\pm$0.195\\ \hline
NAM &0.343$\pm$0.065 &0.120$\pm$0.002&0.045$\pm$0.006&0.493$\pm$0.110&-& 0.082$\pm$0.018\\
ENRL& 0.134$\pm$0.002&0.231$\pm$0.043&0.053$\pm$0.002&0.002$\pm$0.000&-&0.102$\pm$0.032\\\hline
\ours  &1.00$\pm$0.00 &1.00$\pm$0.00 &1.00$\pm$0.00 &0.982$\pm$0.00 &0.974$\pm$0.000 &0.789$\pm$0.011  \\\bottomrule 
\hline
\end{tabular}
\caption{Explanation consistency scores across datasets}\label{tab: consistency}
\end{table}

\begin{table}[ht]
\small
\centering
\begin{tabular}{ccccccc}\toprule
& \ihdp\ & \tcga\ &\ihdpc\ &  \news\ & \eeec\ & \uganda\ \\ \hline 

Model distillation &0.243$\pm$0.126 &0.529$\pm$0.001 & 0.029$\pm$0.003&\textbf{0.712$\pm$0.032} & 0.004$\pm$0.001 & 0.198±0.008\\
Lore & 0.320$\pm$0.084&0.034$\pm$0.013 & 0.030$\pm$0.009&0.142$\pm$0.012&0.002$\pm$0.001 &\textbf{0.265$\pm$0.008} \\
Anchor &0.084$\pm$0.083 &0.125$\pm$0.002 &0.332$\pm$0.016 &0.391$\pm$0.040&0.002$\pm$0.001&0.221$\pm$0.007\\ 
ENRL& 0.452$\pm$0.012&0.512$\pm$0.005&0.032$\pm$0.018&0.053$\pm$0.020&-&0.004$\pm$0.002\\\hline
\midrule
\ours  &\textbf{0.562$\pm$0.056}&\textbf{0.9999$\pm$0.000}&\textbf{0.588$\pm$0.019} &0.697$\pm$0.017&\textbf{0.926$\pm$0.067}& 0.104$\pm$0.011\\\bottomrule 
\hline
\end{tabular}
\caption{Explanation sufficiency scores across datasets (larger score indicates better sufficiency) }\label{tab: tabular_sufficiency}
\end{table}

\subsection{Results for News dataset}\label{sec: news_res}
Table \ref{tab: news_dataset} shows the results for the News dataset.



\begin{table}[!htb]
\small
\centering
    \medskip
    \begin{tabular}{ccc}
    \toprule
\multirow{2}{*} & \multicolumn{1}{c}{\news}
\\ \midrule
&AMSE \\ \hline 
\rowcolor{lightgray} Decision Tree &0.428$\pm$0.051 \\
\rowcolor{lightgray} Random Forest &0.452$\pm$0.048\\
\rowcolor{lightgray} NAM &0.653$\pm$0.026 \\
\rowcolor{lightgray} ENRL &0.638$\pm$0.019 \\
\rowcolor{lightgray} Causal Forest & 0.829$\pm$0.042\\
\rowcolor{lightgray} BART &0.619$\pm$0.040 \\
\hline
\rowcolor{lightgray} \ours (ours) & 0.385$\pm$0.083\\ 
\hline
Dragonnet &-\\ 
TVAE & -\\
TARNet & 0.073$\pm$0.020\\ 
Ganite &-\\ 
DRNet &0.065$\pm$0.021\\ 
VCNet &-\\ 
TransTEE &\underline{0.063$\pm$0.005}\\ \hline
TransTEE + NN & 0.383$\pm$0.041 \\
\ours + TransTEE (ours) &\textbf{0.043$\pm$0.005} & \\ 
\bottomrule 
\hline
\end{tabular}
\caption{ITE estimation errors for the News dataset}
\label{tab: news_dataset}
\end{table}

\subsection{Consistent Ground-truth Outcomes in the \uganda\ Dataset}\label{sec: uganda_consistent}
We observe that in \uganda\ dataset, the ground-truth outcome values are not evenly distributed, which is visually presented in Figure \ref{fig:uganda_outcome}. As this figure suggests, the outcome value of most samples is $-0.8816$ while other outcome values rarely occur. This thus suggests that our method is preferable in such datasets due to its consistent predictions across samples, which can explain the performance gains of \ours\ over baseline methods.
\begin{figure}[h]
    \centering
\includegraphics[width=\textwidth]{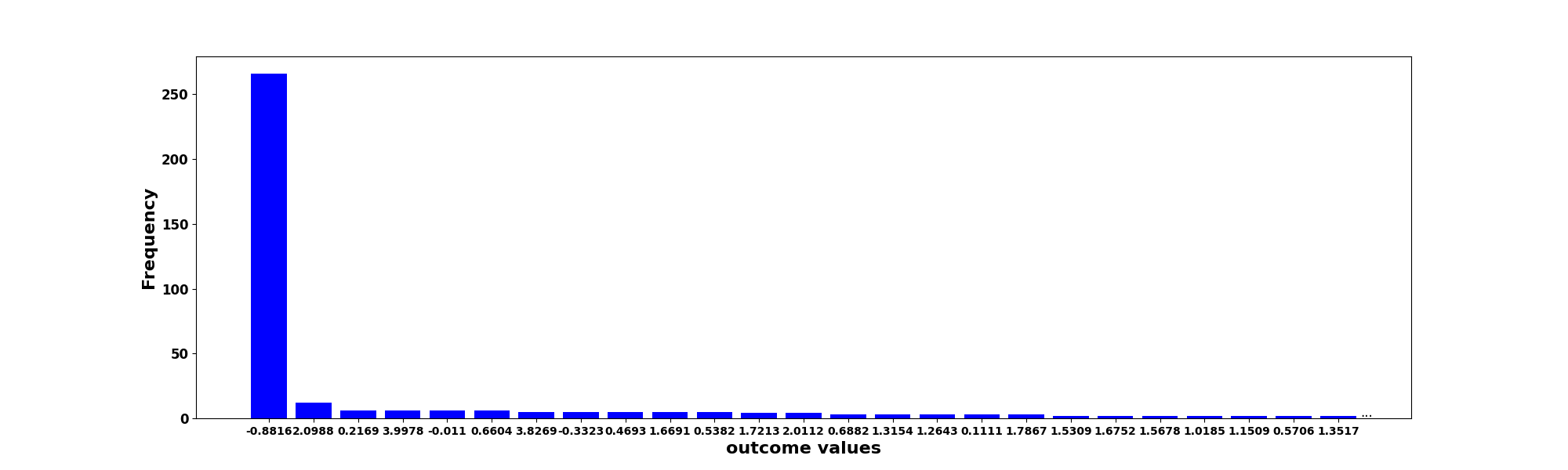}
    \caption{Frequency of the outcome values on \uganda\ dataset}
    \label{fig:uganda_outcome}
\end{figure}

\section{Additional Qualitative Analysis}\label{sec: exp_analysis}
As shown in Figure \ref{fig:iclr_example}, \ours\ generates one rule for one example image from \uganda\ dataset, which is defined on two concepts, i.e., one type of patches mainly containing reddish pink pixels that represent ``soil moisture content'' and the other type of patches mainly comprised of brown pixels indicating little soil. This rule thus represents the images from one type of location where there is plenty of soil moisture content that is suitable for agricultural development. Therefore, after the government grants are distributed in such areas, a more significant treatment effect is observed, i.e., 0.65. This is an indicator of significantly increasing working hours on the skilled jobs by the laborers in those areas. This is consistent to the conclusions from \citep{JJD-Heterogeneity, JJD-Confounding} which states that government grant support is more useful for areas with more soil moisture content.

\section{Feature Extraction from Image Data}\label{sec: feature_extraction}

To extract concepts from images of \uganda\ dataset, we segment each image as multiple superpixels \citep{achanta2012slic}, embed those superpixels with pretrained clip models \citep{radford2021learning}, and then perform K-means on these embeddings. Each of the resulting cluster centroids is regarded as one concept and we count the occurrence of each concept as one feature for an image. Specifically, we extract 20 concepts from the images of \uganda\ dataset, which are visually presented in Figure \ref{fig: image_concepts}.

\begin{figure}[!h]
    \centering
    \includegraphics[width=\textwidth]{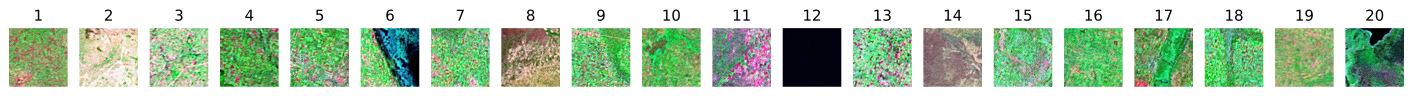}
    \caption{Extracted concepts from \uganda\ dataset}
    \label{fig: image_concepts}
\end{figure}

Various patterns of image patches are captured by Figure \ref{fig: image_concepts}. For example, patch 12 is almost all black, which represents the areas with water, say, river areas or lake areas. Also, as mentioned in Figure \ref{fig:iclr_example}, patch 11 with reddish pink pixels represents ``soil moisture content'', which is an important factor for determining whether to take interventions in the anti-poverty program conducted in Uganda. 

\end{document}